\documentclass{article}
\usepackage[final]{neurips_2025}
\usepackage[utf8]{inputenc} 
\usepackage[T1]{fontenc}    
\usepackage{hyperref}       
\usepackage{url}            
\usepackage{booktabs}       
\usepackage{amsfonts}       
\usepackage{amsmath}        
\usepackage{nicefrac}       
\usepackage{microtype}      
\usepackage{xcolor}         
\usepackage{graphicx}
\usepackage{amsfonts,amssymb,amsthm}
\usepackage{capt-of}         
\usepackage{wrapfig}            
\usepackage{caption}            
\usepackage{wrapfig}       
\usepackage{etoolbox}      
\usepackage{algorithm}
\usepackage{algorithmic}

\newtheorem{assumption}{Assumption}[section]
\newtheorem{theorem}{Theorem}[section]

\title{Direct Fisher Score Estimation for Likelihood Maximization}

\author{
  Sherman Khoo \thanks{School of Mathematics, University of Bristol}
  \And
  Yakun Wang \footnotemark[1]
  \And 
  Song Liu \footnotemark[1]
  \And
  Mark Beaumont \thanks{School of Biological Sciences, University of Bristol\\[0.2em] Correspondence to: Sherman Khoo <sherman.khoo@bristol.ac.uk>}
}

\begin{document}

\maketitle

\begin{abstract}
We study the problem of likelihood maximization when the likelihood function is intractable but model simulations are readily available. We propose a sequential, gradient-based optimization method that directly models the Fisher score based on a local score matching technique which uses simulations from a localized region around each parameter iterate. By employing a linear parameterization for the surrogate score model, our technique admits a closed-form, least-squares solution. This approach yields a fast, flexible, and efficient approximation to the Fisher score, effectively smoothing the likelihood objective and mitigating the challenges posed by complex likelihood landscapes. We provide theoretical guarantees for our score estimator, including bounds on the bias introduced by the smoothing. Empirical results on a range of synthetic and real-world problems demonstrate the superior performance of our method compared to existing benchmarks.
\end{abstract}

\section{Introduction}
\label{sec:intro}

Implicit simulator-based models are now routine in many scientific fields, such as biology \citep{b416fef4a8b244fc872ac0ec03971993}, cosmology \citep{article3}, neuroscience \citep{Sterratt}, engineering \citep{bharti2021general}, and other scientific applications \citep{article}. In traditional statistical models, there is a prescribed probabilistic model, which provides an explicit parameterization for the data distribution, allowing development of the likelihood function for further inference. In contrast, simulation-based models define the distribution implicitly through a computational simulator. Thus, while simulation of the data for various parameter settings is possible, the probability density function of the data or the likelihood function is often unavailable in closed form. This problem setting is known as likelihood-free inference or simulation-based inference (SBI) \citep{Cranmer_2020}. 

Traditionally, methods in this setting have focused on a Bayesian inference technique known as approximate Bayesian computation (ABC) \citep{beaumont2002approximate}. Fundamentally, the ABC method builds an approximation to the Bayesian posterior distribution by drawing parameter samples from the prior distribution, generating datasets from the drawn parameter values, and filtering parameter values through a rejection algorithm based on the distance of the generated summary statistic of the dataset from the observations. In recent years, there has been a rise of generative modeling or unsupervised learning in machine learning, which aims to recover a data distribution given a set of samples. Such generative models are often built from neural networks \citep{Mohamed_Lakshminarayanan_2017}, and given the fundamental similarity with the SBI problem setting, this has led to significant cross-pollination across the two fields, with the development of many SBI inference methods using generative neural networks \citep{Durkan_Papamakarios_Murray_2018, Papamakarios_Pavlakou_Murray_2018}. 

While significant progress in SBI has come from Bayesian approaches, such methods are often computationally demanding; furthermore, many scientific disciplines retain a preference towards maximum likelihood approaches. In contrast, practitioners who favor faster point estimation through maximum likelihood lack comparably mature tools. To close this gap, we propose a fast, simulation-efficient, and robust gradient-based technique for estimating the maximum likelihood for SBI. We build on a popular technique in generative modeling known as score matching \citep{hyvarinen2005estimation}, which has seen significant use in score-based generative models \citep{Song_Ermon_2020}, now a cornerstone for many state-of-the-art approaches in generative modeling \citep{yang2024diffusion}. We adapt score matching to estimate the Fisher score, that is, the gradient of the log-likelihood function with respect to the parameters, within a localized region. This estimated gradient can then be used in any first-order gradient-based stochastic optimization algorithm such as stochastic gradient descent (SGD) to obtain an approximate maximum likelihood estimator (MLE) and serves as a potential avenue for uncertainty quantification for the MLE through the empirical Fisher information matrix.

Our contributions in this work are as follows.
\begin{itemize}
    \item We propose a lightweight, simulation-efficient, and robust method for maximum likelihood estimation of simulator models based on a novel local Fisher score matching technique
    \item We derive theory for our local Fisher score matching technique and establish a connection with the Gaussian smoothing gradient estimator, offering a unifying perspective for zeroth-order optimization techniques and likelihood optimization for SBI
    \item We demonstrate the effectiveness of our method in real-world experiments for applied machine learning and cosmology problems, showcasing both its efficiency and robust performance compared to existing approaches
\end{itemize}

\section{Background}
\label{sec:background}

\subsection{Score Matching}
\label{subsec:bg_sm}

Density estimation is the problem of learning a data distribution \(p_{D}(\mathbf{x}) \) using only an observed dataset, \(\mathbf{x} \sim p_{D} \). An approach to this problem is learning the density with an energy-based model (EBM), which parameterizes the model through its scalar-valued energy function \(E_\theta : \mathbb{R}^k \to \mathbb{R} \) where \(\mathbf{x} \in \mathcal{X} \subseteq \mathbb{R}^k \), giving the model density \( p_\theta(\mathbf{x}) = \exp (-E_\theta(\mathbf{x})) / Z_\theta \).

Since the energy function is an unnormalized density function, it can be flexibly parameterized, usually through a neural network. However, note that the normalization constant \(Z_\theta = \int \exp (-E_\theta(\mathbf{x})) d\mathbf{x}\) is still a function of \(\theta\), and therefore will still need to be computed in training the EBM through standard likelihood maximization. Since this multidimensional integral is often intractable and requires a costly approximation method, score matching \citep{hyvarinen2005estimation} is often used to bypass the computation of the normalization constant. This is done by considering an alternate training objective instead of MLE, based on the score function  \(s_\theta : \mathbb{R}^k \to \mathbb{R}^k \), where \( s_\theta = \nabla_\mathbf{x} \log p_\theta = - \nabla_\mathbf{x} E_\theta\). In fact, since equivalence in the score amounts to equivalence in the distribution, matching the scores is equivalent to performing density estimation. One starting point is the explicit score matching objective (ESM). Defining the gradient operator on a scalar-valued function as \(\nabla_\mathbf{x} := (\frac{\partial}{\partial x_1}, \ldots, \frac{\partial}{\partial x_k})^\top\), where \(\frac{\partial}{\partial x_i}\) is the partial derivative operator for \(\mathbf{x} = (x_1,\ldots,x_k) \), and the Jacobian operator on a vector-valued function \(\mathbf{f} : \mathbb{R}^m \to \mathbb{R}^n\) as \(\mathbf{J}_{i,j} = [\frac{\partial f_i}{\partial x_j}]_{i,j}\), we have: \[
\mathcal{L}_{\textrm{ESM}}(\theta) = \mathbb{E}_{\mathbf{x} \sim p_D(\mathbf{x})} \frac{1}{2} \|s_\theta (\mathbf{x}) - \nabla_{\mathbf{x}} \log p_D(\mathbf{x}) \|^2
\] 

However, this objective is not tractable due to the need to evaluate \(\nabla_{\mathbf{x}} \log p_D(\mathbf{x}) \). Hence, this objective is transformed to:
\[
\mathcal{L}_{\textrm{ESM}}(\theta) = \mathbb{E}_{\mathbf{x} \sim p_D(\mathbf{x})}\left[\frac{1}{2}\left\|s_{\mathbf{\theta}}(\mathbf{x})\right\|^2+\operatorname{tr}\left(\mathbf{J}_{\mathbf{x}} s_{\theta}(\mathbf{x})\right)\right] + ( \text{constants w.r.t. } \theta)
\]  

Although this objective can be directly estimated, and thus optimized and used in the training of an EBM, it is computationally expensive due to the presence of the Jacobian term, motivating further extensions to the standard score matching objective, such as the denoising score matching objective \citep{Vincent_2011} and the sliced score matching objective \citep{song2019sliced}.

\subsection{Maximum Likelihood Estimation and Fisher Score}
\label{subsec:bg_mle}

Maximum likelihood estimation (MLE) is a foundational tool in statistical inference, under standard regularity conditions, it is consistent and asymptotically efficient \citep[Section 10]{casella2024statistical}. Central to the MLE is the Fisher score, defined as the gradient of the log-likelihood with respect to the parameters, \(\nabla_\theta \log p(\mathbf{x} \mid \theta )\). From an optimization point of view, the score provides the direction of steepest ascent of the log-likelihood in parameter space, and thus drives gradient-based MLE approaches. From an inferential point of view, the covariance of the Fisher score is equal to the Fisher information matrix (FIM), which, through the Cramér–Rao lower bound \citep{rao1992information}, lower bounds the variance of any unbiased estimator. Furthermore, the distribution of the MLE is asymptotically normal with covariance equal to the inverse of the FIM, which underpins Wald-type confidence intervals and hypothesis tests \citep[Section 5]{van2000asymptotic}.

\subsection{Notation and Problem Setup}

We consider a statistical model where the data \( \mathbf{x} \in \mathcal{X} \subset\mathbb{R}^{d_\mathbf{x}} \) are generated from a distribution \( P_\theta \) parameterized by \( \theta \in \Omega \subset \mathbb{R}^{d_\theta} \). In the simulation-based inference setting, this statistical model is implicitly defined, so we can draw samples from this model for any choice of \( \theta \) but the closed-form expression for the probability density function, and hence the likelihood function is not known.

Given a set of \( N \) independent and identically distributed observations, \( \mathcal{D} = \{ \mathbf{x}_{i} \}_{i=1}^N \), drawn from the true data-generating process \( \mathbf{x}_{i} \sim P_{\theta^*} \), where \( \theta^* \) denotes the true parameter, the maximum likelihood estimator is \( \hat{\theta}_{\mathrm{MLE}} = \arg\max_\theta \, p(\mathcal{D} \mid \theta) \). 

As the likelihood function \( L(\theta; \mathcal{D}) = \prod_{i=1}^N p(\mathbf{x}_{i} \mid \theta) \) is not available for SBI models, typical likelihood maximization cannot be applied directly. We thus propose a Fisher score matching-based estimator, \( \hat{\theta}_{\mathrm{FSM}} \). Our method is fundamentally a first-order optimization approach, and our main focus is on the direct estimation of the gradient of the log-likelihood function at each parameter iteration, which is done with a novel local Fisher score matching objective. We first discuss our Fisher score estimation technique in Section~\ref{sec:fisher_score_matching}, before proceeding with the MLE procedure in Section~\ref{sec:FSM-MLE}.

\section{Likelihood-free Fisher Score Estimation}
\label{sec:fisher_score_matching}

Score matching \citep{hyvarinen2005estimation} is a classical method in density estimation, but is not directly applicable in likelihood gradient maximization, as it typically targets the Stein score, i.e., the gradient with respect to the data \( \nabla_\mathbf{x} \log p_\theta(\mathbf{x}) \) instead of the Fisher score, which is the gradient with respect to the parameters \( \nabla_\theta \log p_\theta(\mathbf{x}) \). Hence, we propose to adapt score matching into a novel local Fisher score estimation technique which estimates the gradient of the log-likelihood for a fixed parameter point \( \theta_t \) at any data sample \( \mathbf{x} \), \(\nabla_\theta \ell(\theta; \mathbf{x})\big|_{\theta=\theta_t} 
\;=\; \nabla_\theta \log p\bigl(\mathbf{x} \mid \theta\bigr)\Big|_{\theta=\theta_t}\).

\subsection{Local Fisher Score Matching Objective}
\label{subsec:fsm_local}

Around the target parameter point \( \theta_t \), we introduce a local proposal distribution \( q(\theta \mid \theta_t) \), which we typically take as an isotropic Gaussian distribution, \( q(\theta \mid \theta_t) = \mathcal{N}(\theta_t, \sigma^2 I)\). When combined with the statistical model \( P_\theta \), we induce a joint distribution in both the data and parameter space that has probability density \( p(\mathbf{x} \mid \theta) q(\theta \mid \theta_t) \). Note that by drawing parameter samples from the local proposal distribution and then drawing corresponding data samples for the parameter samples, we can easily draw samples from this joint distribution. 

To estimate the score function, we use a score model \(S_{W} : \mathbb{R}^{d_\mathbf{x}} \to \mathbb{R}^{d_\theta} \), where \( S_{W}(\mathbf{x}) \) has parameters \( W \). Our starting point is the adapted, localized score matching least-squares loss for the Fisher score.

\begin{align}
\label{eq:lsm_def}
  \mathcal{J}(W; \theta_t) \;=\;
  \mathbb{E}_{\mathbf{x}\sim p(\mathbf{x} \mid \theta), \theta \sim q(\theta \mid \theta_t)}
  \Bigl[
    \bigl\|\nabla_{\theta}\log p(\mathbf{x} \mid \theta)\,-\,S_{W}(\mathbf{x})\bigr\|^{2}
  \Bigr]
\end{align}

As we are within the simulation-based inference framework, we do not have a closed form expression for the Fisher score \(\nabla_{\theta}\log p(\mathbf{x} \mid \theta) \) and hence this objective function is not tractable. We first expand the square of Equation~\eqref{eq:lsm_def}, which allows us to rewrite \( \mathcal{J}(W; \theta_t) \) as:

\[
  \mathcal{J}(W; \theta_t) =
  \mathbb{E}_{\mathbf{x}\sim p(\mathbf{x} \mid \theta), \theta \sim q(\theta \mid \theta_t)}
  \Bigl[
      \bigl\|S_{W}(\mathbf{x})\bigr\|^{2}
      \;-\;
      2\,S_{W}(\mathbf{x})^{\top}
      \,\nabla_{\theta}\log p(\mathbf{x} \mid \theta)
    \Bigr] + ( \text{constants w.r.t. } W)
\]

We focus on the cross-term, \( \mathbb{E}_{\mathbf{x}\sim p(\mathbf{x} \mid \theta), \theta \sim q(\theta \mid \theta_t)} \bigl[S_{W}(\mathbf{x})^{\top} \nabla_{\theta}\log p(\mathbf{x} \mid \theta)\bigr] \). Using an integration-by-parts trick, this term can be transformed to \( - \mathbb{E}_{\mathbf{x}\sim p(\mathbf{x} \mid \theta), \theta \sim q(\theta \mid \theta_t)} \Bigl[S_{W}(\mathbf{x})^{\top} \nabla_{\theta}\log q(\theta \mid \theta_t)\Bigr] \). Note that we have eliminated the dependence on the intractable likelihood function \( \log p(\mathbf{x} \mid \theta)  \). Thus, this allows us to rewrite \( \mathcal{J}(W; \theta_t) \) as follows.

\begin{theorem}[Local Fisher Score Matching (FSM)]
\label{thm:lsm}

Let \(\mathcal{J}(W)\) be defined as in Equation~\eqref{eq:lsm_def}. Under suitable boundary conditions, it can be rewritten (up to an additive constant w.r.t.\ \( W \)) as
\begin{align}
  \mathcal{J}(W; \theta_t) 
  \;=\;
  \mathbb{E}_{\mathbf{x}\sim p(\mathbf{x} \mid \theta), \theta \sim q(\theta \mid \theta_t)}
  \Bigl[
      \bigl\|S_{W}(\mathbf{x}) \bigr\|^{2}
      \;+\;
      2\,S_{W}(\mathbf{x})^{\top}
      \,\nabla_{\theta}\log q(\theta \mid \theta_t)
    \Bigr]
  \label{eq:lsm_simplified}
\end{align}
\end{theorem}

The complete details for Theorem~\ref{thm:lsm} are provided in Appendix~\ref{appendix:lsm_proof}. Given that we can draw proposal samples \( \{ \theta^{(j)} \}_{j=1}^m \) where \( \theta^{(j)} \sim q(\theta \mid \theta_t) \) and corresponding data samples \( \{ \mathbf{x}^{(j)}_k \}_{k=1}^n \) where \( \mathbf{x}^{(j)}_k \sim p(\mathbf{x} \mid \theta^{(j)})\), the objective \( \mathcal{J}(W; \theta_t) \) can be approximated by Monte Carlo estimation.

\begin{equation}
    \hat{\mathcal{J}}(W; \theta_t) = \frac{1}{m} \sum_{j=1}^m \frac{1}{n} \sum_{k=1}^n \Bigl[
        \bigl\|S_{W}(\mathbf{x}^{(j)}_k)\bigr\|^{2}
      \;+\;
      2\,S_{W}(\mathbf{x}^{(j)}_k)^{\top}
      \,\nabla_{\theta}\log q(\theta \mid \theta_t)|_{\theta = \theta^{(j)}} \Bigr]
    \label{eq:lsm_est}
\end{equation}

Next, we show the optimal solution for the local FSM objective, \( \mathcal{J}(W; \theta_t) \). The proof of Theorem~\ref{thm:bo-fsm} in Appendix~\ref{appendix:bo-proof}. 

\begin{theorem}[Bayes-optimal Local Fisher Score]
\label{thm:bo-fsm}

    The optimal score model for the FSM objective \( \mathcal{J}(W; \theta_t) \), is
    \[
    S^*(\mathbf{x}; \theta_t) = \mathbb{E}_{\theta \sim p(\theta \mid \mathbf{x}, \theta_t)} \bigl[\nabla_\theta \log p(\mathbf{x} \mid \theta)\bigr]
    \]
\end{theorem}

As the score matching objective Equation~\eqref{eq:lsm_def} is taken as an expectation over the parameter proposal distribution \( q(\theta \mid \theta_t) \), the Bayes-optimal score model for this objective is generally biased and instead of being the true score at the point \( \theta_t \), it is instead an average of the score over the posterior induced from the proposal distribution and the statistical model, that is, \( p(\theta \mid \mathbf{x}, \theta_t) \). Thus, this score matching objective targets a smoothed likelihood around \( \theta_t \). We elaborate on this in more detail in Section~\ref{sec:gaussian-smoothing}.

\subsection{Score Model Parameterization}
\label{subsec:score_model_param}

A key aspect of the Fisher score matching technique is the choice of parameterization for the surrogate score model, \( S_{W}(\mathbf{x}; \theta_t) \), which approximates the Fisher score at the target parameter iterate \( \theta_t \), \( \nabla_\theta \log p_\theta(\mathbf{x})|_{\theta = \theta_t} \). For computational tractability, we propose using a lightweight linear surrogate score model based on the following derivation.

Let the surrogate score model be defined as \( S_W(\mathbf{x}; \theta_t) = W^\top \mathbf{x} \), where \( W \in \mathbb{R}^{d_\mathbf{x} \times d_\theta} \) is the weight matrix for our model. Recall that we first draw a set of parameters \( \{\theta^{(j)} \}_{j=1}^m \) from the proposal distribution \( q(\theta \mid \theta_t) \). Then, define the \( j\)-th data matrix as \( X_j \in \mathbb{R}^{n \times d_\mathbf{x}} \) constructed from \(n \) training samples \( \{ \mathbf{x}^{(j)}_k \}_{k=1}^n \) drawn from the model \( p(\mathbf{x} \mid \theta) \) at \( \theta^{(j)} \), and the \(j\)-th Gram matrix as \( G_j = X_j^\top X_j\). Using the linear score model for the local Fisher score matching objective function in Equation~\eqref{eq:lsm_est} and solving for the first-order conditions, we obtain the normal equation,

\begin{equation}
    (\sum_{j=1}^m G_j) \hat{W} = - \sum_{j=1}^m \sum_{k=1}^n \bigl[\mathbf{x}^{(j)}_k \nabla_{\theta}\log q(\theta \mid \theta_t)|_{\theta = \theta^{(j)}} ^\top \bigr]
    \label{eq:normal_lsm}
\end{equation}

We can thus obtain a closed-form solution for the linear Fisher score matching estimator as:
\begin{equation}
    \hat{W} =  - (\sum\limits_{j=1}^{m} G_j)^{-1} \sum\limits_{j=1}^{m} \sum\limits_{k=1}^{n} \bigl[ \mathbf{x}^{(j)}_k \nabla_{\theta}\log q(\theta \mid \theta_t)|_{\theta = \theta^{(j)}}^\top \bigr]
    \label{eq:linear_fish}
\end{equation} 

Once \( \hat{W} \) is obtained, we can use this to construct our Fisher score estimator \( \hat{S}(\mathbf{x}; \theta_t) = \hat{W}^\top \mathbf{x} \). We provide a complete derivation and further discussion in the Appendix~\ref{appendix:linear_model}. Although the local Fisher score matching objective is a general framework that is agnostic to the choice of parameterization of the model, using a linear model essentially recasts the model estimation procedure as multivariate linear regression, benefiting from well-understood theory and efficient implementations. Although a linear model might not be sufficient to fully capture the full data-parameter relationship, it provides a strong baseline that we find works well empirically compared to a more flexible neural network-based model, which incurs significant computational costs in the form of an inner optimization loop and increased variance. We provide empirical comparisons with the neural network-based score model in the relevant experimental sections of the appendix, and details of the implementation in Appendix~\ref{appendix:nn_model}.

\section{Likelihood-free MLE with Approximate Fisher Score}
\label{sec:FSM-MLE}

Using our local Fisher score matching (FSM) method as described in Section~\ref{sec:fisher_score_matching}, we describe how maximum likelihood estimation (MLE) can be performed in the likelihood-free setting. Unlike many SBI methods that attempt to estimate the likelihood globally, our method is inherently sequential by focusing only on a local Fisher score estimation at the parameter point \( \theta_t \). 

Given a set of \( N \) independent and identically distributed observations, \( \mathcal{D} = \{ \mathbf{x}_{i} \}_{i=1}^N \), at a fixed parameter point \( \theta_t \), we obtain an estimated FSM model \( \hat{S}(\mathbf{x}; \theta_t) \) using training samples \( \{\theta^{(j)}\}_{j=1}^m, \{ \mathbf{x}^{(j)}_k\}_{k=1}^n \), drawn from \( \theta^{(j)} \sim q(\theta \mid \theta_t) \), \( \mathbf{x}^{(j)}_k \sim p(\mathbf{x} \mid \theta^{(j)}) \). As the FSM model is a function of \( \mathbf{x} \), we can evaluate it at any observation \( \mathbf{x}_i \), providing us with an approximate gradient of the log-likelihood evaluated at \(\theta_t \), \( \hat{\nabla}_\theta \ell(\theta_t; \mathcal{D}) = \sum_{i=1}^N \hat{S}(\mathbf{x}_{i}; \theta_t) \). This can then be used directly in any iterative stochastic gradient-based algorithm such as stochastic gradient descent (SGD) \citep{robbins1951stochastic}, Adam \citep{kingma2017adam}, or RMSProp \citep{tieleman2012lecture}, where at each parameter iteration \( \theta_t \), a new FSM model \( \hat{S}(\mathbf{x}; \theta_t) \) is estimated. The FSM-MLE algorithm with SGD is presented in Algorithm~\ref{alg:cap}

\begin{algorithm}[H]
\caption{FSM-MLE Algorithm (SGD)}\label{alg:cap}
\begin{algorithmic}
\STATE Input: \( N \) independent and identically distributed observations \(\mathcal{D} = \{\mathbf{x}_i \}_{i=1}^{N}\), initial parameter \(\theta_0\), step size \(\eta\), and proposal distribution \(q(\theta \mid \theta_t)\)
\STATE Initialize \(t \leftarrow 0\)
\WHILE{\(t < T\)}
    \STATE 1. For current iterate \(\theta_t\), sample \(\{\theta^{(j)}\}_{j=1}^m\) from proposal distribution \(\theta \sim q(\theta \mid \theta_t)\), and then sample corresponding data samples \( \{\mathbf{x}^{(j)}_k \}_{k=1}^n \), from \( \mathbf{x}^{(j)}_k \sim p(\mathbf{x} \mid \theta^{(j)}) \). 
    \STATE 2. Estimate Fisher score model \( \hat{S}(\mathbf{x}; \theta_t) \) using training samples \( (\{\theta^{(j)}\}_{j=1}^m, \{\mathbf{x}^{(j)}_k \}_{k=1}^n) \)
    \STATE 3. Set \(\theta_{t+1} \leftarrow \theta_t + \eta \hat{S}(\mathcal{D}; \theta_t) \), where \( \hat{S}(\mathcal{D}; \theta_t) = \sum_{i=1}^N \hat{S}(\mathbf{x}_i;\theta_t)\)
    \STATE 4. \(t \leftarrow t + 1\)
\ENDWHILE
\end{algorithmic}
\end{algorithm}

\subsection{Fisher Score Proposal Distribution}
\label{subsec:choice-of-proposal}

Our local FSM approach crucially uses a proposal distribution \( q(\theta \mid \theta_t) \) in the parameter space, defining a local region for the estimation of our Fisher score model. Although most distributions with a differentiable and unbounded density can be used, we use an isotropic Gaussian distribution \( q(\theta \mid \theta_t) = \mathcal{N}(\theta \mid \theta_t, \sigma^2 I) \), which has a simple, closed-form solution and direct theoretical interpretation as discussed in Section~\ref{sec:theory}. This introduces a single scalar hyperparameter \( \sigma \) that controls the width of the proposal distribution. While further extensions such as a diagonal covariance matrix or an adaptive covariance could be explored, we keep to the isotropic Gaussian proposal distribution as it provides a simple and effective baseline. We provide further discussion on the choice of proposal distribution and a calibration scheme for \( \sigma \) in Appendix~\ref{appendix:fsm_proposal}.

\section{Theoretical Analysis}
\label{sec:theory}
In this section, we provide a theoretical analysis of our proposed local Fisher score matching technique and the stochastic gradient optimization based on this technique.

\subsection{Connection to Gaussian Smoothing}
\label{sec:gaussian-smoothing}

Gaussian smoothing is a popular zeroth-order optimization technique that estimates gradients using only function evaluations when the gradient function is not known \citep{nesterov2017random, duchi2012randomized}. As the Gaussian smoothing gradient estimator targets a smoothed function, it is widely applicable even for non-smooth functions, which would not be amenable with standard gradient estimation, and has been shown to be robust to local optima \citep{starnes2023gaussian} and applicable for many challenging machine learning problems \citep{salimans2017evolution}. 

Although standard Gaussian smoothing is straightforward for black-box optimization problems, note that it is not directly applicable in the simulation-based inference setting as the intractable likelihood \( L(\theta) = p(\mathbf{x} \mid \theta) \) is not explicitly accessible. Nonetheless, we show here that our proposed local Fisher score matching technique can be directly cast as a \emph{likelihood-free} analogue of Gaussian smoothing. Specifically, under a Gaussian proposal distribution, \(q(\theta \mid \theta_t) = \mathcal{N}(\theta \mid \theta_t, \sigma^2 I)\), the Bayes-optimal Fisher score is exactly the gradient of a smoothed likelihood. We provide a full proof of Theorem~\ref{thm:gs} in Appendix~\ref{appendix:gs_proof}.

\begin{theorem}[Equivalence as Gaussian Smoothing]
    \label{thm:gs}
    Under an isotropic Gaussian proposal, \(q(\theta \mid \theta_t) = \mathcal{N}(\theta \mid \theta_t, \sigma^2 I)\), the optimal FSM estimator is equivalent to the gradient of the smoothed likelihood
    \[
    \nabla_{\theta_t} \tilde{\ell}(\theta_t ; \mathbf{x}) = \mathbb{E}_{\theta \sim p(\theta \mid \mathbf{x}, \theta_t)} \nabla_\theta \log p(\mathbf{x} \mid \theta)
    \]

    where \(\tilde{\ell}(\theta_t;\mathbf{x})
  \;=\;
  \log \int 
  p\bigl(\mathbf{x} \mid \theta\bigr)\,q\bigl(\theta \mid \theta_t\bigr)
  \,d\theta\) and \( p(\theta \mid \mathbf{x}, \theta_t) \propto p(\mathbf{x} \mid \theta) q( \theta \mid \theta_t) \) is the induced posterior from the proposal distribution \( q(\theta \mid \theta_t) \)
  
\end{theorem}

\begin{wrapfigure}[23]{r}{0.5\textwidth} 
    \captionsetup{skip=0.1pt}
    \centering
    \includegraphics[width=\linewidth]{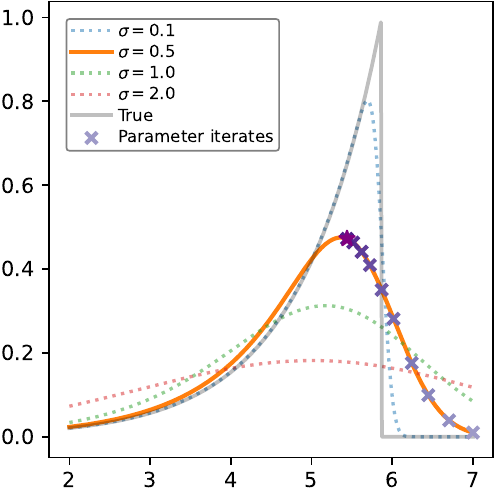} 
    \caption{Optimizing a non-smooth, exponential likelihood with FSM estimator (\(\sigma = 0.5\)) for 10 parameter iterates from initial point \( \theta_0 = 7\)} 
    \label{fig:shifted-exp}
\end{wrapfigure}

Observe that the smoothed likelihood can be further rewritten as
\[
  \tilde{\ell}(\theta_t;\mathbf{x})
  \;=\;
  \log 
  \mathbb{E}_{\mathbf{z} \sim \mathcal{N}(0,I)}
    \Bigl[L(\theta_t + \sigma\,\mathbf{z}; \mathbf{x})\Bigr].
\]

where \( L(\theta;\mathbf{x}) = p(\mathbf{x} \mid \theta) \). This is exactly the Gaussian-smoothed likelihood function, except importantly that \emph{explicit evaluations of the likelihood \( L(\theta) \) were not used}. Instead, our Fisher score matching technique only obtains samples from the model \( p( \mathbf{x} \mid \theta) \) for the FSM estimation. Hence, our method directly inherits many of the robustness benefits of Gaussian smoothing while still being applicable in the SBI setting. 

Figure~\ref{fig:shifted-exp} demonstrates the effects of smoothing in a one-dimensional, shifted exponential likelihood model with a single observation. The true likelihood is zero for \( \theta \geq \hat{\theta}_{\mathrm{MLE}} \), and hence any gradient-based optimization which is initialized beyond the boundary will be stuck in that region. However, using our smoothed likelihood (depicted with differing values of proposal variance \( \sigma^2 \)), we are able to obtain a non-zero gradient even outside the nominal support, allowing us to successfully optimize the likelihood function.

We can also further view the FSM procedure as a form of Empirical Bayes (EB) \citep{morris1983parametric}, by interpreting the proposal distribution \( q(\theta \mid \theta_t) \) as a local prior centered at \( \theta_t \), which, together with the simulator model, defines an EB marginal likelihood function \( \tilde{\ell}(\theta_t;\mathbf{x}) \). Theorem~\ref{thm:gs} then shows that our Bayes-optimal FSM estimator is exactly the hyperparameter gradient of the EB marginal likelihood. Hence, this provides a complementary Bayesian interpretation of our FSM method in addition to the optimization viewpoint of Gaussian smoothing.

\subsection{Properties of the FSM estimator}
\label{subsec:fsm_properties}

We now provide theoretical guarantees for our FSM estimator under a Gaussian proposal distribution by characterizing its bias. In particular, by establishing the bias in terms of the smoothing hyperparameter \( \sigma \), we highlight a fundamental trade-off in the FSM estimation procedure.

\begin{theorem}[Bias characterization of the FSM estimator]
\label{thm:bias_variance}
Let \( \theta^* \) be the true parameter, and denote \( \mathbf{x}_0 \sim P_{\theta^*} \) as random observations sampled from the true model. Suppose there exists a unique maximum likelihood estimator for this model, and that the log-likelihood is \(L\)-smooth. Recall that \( g(\mathbf{x}_0; \theta_t) = \nabla_\theta \log p(\mathbf{x}_0 \mid \theta)|_{\theta = \theta_t} \) is the true Fisher score, \( S^*(\mathbf{x}_0; \theta_t) = \mathbb{E}_{\theta \sim p(\theta \mid \mathbf{x}, \theta_t)} \nabla_\theta \log p(\mathbf{x} \mid \theta) \) is the optimal FSM estimator. For a fixed parameter point \( \theta_t \),

  The bias at \(\theta_t\) is bounded by
  \[
    \mathbb{E}_{\mathbf{x}_0}\!\bigl\|S^{*}(\mathbf{x}_0;\theta_t)-g(\mathbf{x}_0;\theta_t)\bigr\|
    \;\le\;
    L\; \sqrt{d}\,\sigma\, \mathbb{E}_{\mathbf{x}_0} [R(\mathbf{x}_0)]
  \]

where \( R(\mathbf{x}) = \frac{p(\mathbf{x} \mid \theta^*)}{ p(\mathbf{x} \mid \theta_t)} \) is a likelihood ratio term and \( d \) is the dimension of the parameter space 

\end{theorem}

\begin{wrapfigure}[14]{r}{0.5\textwidth} 
    \captionsetup{skip=0.1pt}
    \centering
    \includegraphics[width=\linewidth]{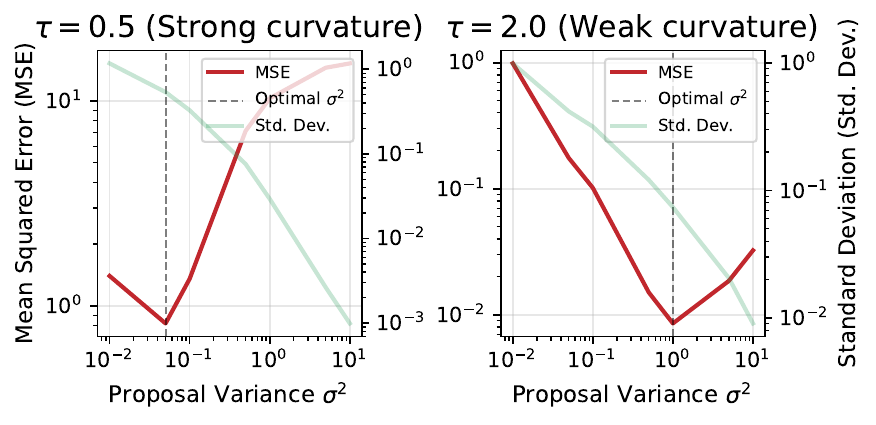}  
    \caption{Mean-squared error and standard deviation of the FSM estimator \( \hat{S} \) with varying proposal hyperparameter \( \sigma^2 \), for two Gaussian likelihoods \( \mathbf{x} \sim \mathcal{N}(\theta, \tau^2) \) with differing curvature.} 
    \label{fig:optimal-sigma}
\end{wrapfigure}

We provide a full proof in Appendix~\ref{appendix:FSM-BV}. From Theorem~\ref{thm:bias_variance}, we can see that, increasing \( \sigma \), we increase the bias of the FSM estimator. Intuitively, this is because \( \sigma \) governs the degree of smoothing, which induces a "smearing" effect of the FSM gradient estimates. On the other hand, for the linear FSM estimator, note that in the estimator \( \hat{W} \), we have the proposal gradient term \( \nabla_{\theta}\log q(\theta \mid \theta_t)|_{\theta = \theta^{(j)}} = - \frac{1}{\sigma^2} (\theta^{(j)} - \theta_t) \), and hence taking \( \sigma \to 0 \) inflates the variance of \( \hat{W} \). Thus, there is a fundamental bias-variance trade-off in the choice of \( \sigma \). 

Figure~\ref{fig:optimal-sigma} empirically illustrates the bias-variance trade-off of the linear FSM estimator with an isotropic Gaussian proposal distribution for two Gaussian likelihood models with differing curvature. In particular, the figure also shows the effect of the log-likelihood curvature, or the gradient-Lipschitz constant \( L \) from Theorem~\ref{thm:bias_variance}, on the MSE-optimal choice of the proposal scale \( \sigma \). When the curvature is stronger (larger \( L \)), smoothing tends to introduce more bias, and the optimal \( \sigma \) is smaller to control the bias. Conversely, when curvature is weaker (smaller \(L \)), a larger \( \sigma \) is optimal to reduce the variance of the score estimator. 

Furthermore, note that the likelihood ratio term, \( R(\mathbf{x}) = \frac{p(\mathbf{x} \mid \theta^*)}{ p(\mathbf{x} \mid \theta_t)} \) encodes the estimation error from using training samples around the parameter iterate points \( \theta_t \) to estimate an FSM estimator that is evaluated at observations \( \mathbf{x}_0 \sim P_{\theta^*} \). Hence, for parameter iterates \( \theta_t \) that are far from the true parameter \( \theta^* \), we are likely to get a subpar estimation of the true gradient, while as we approach the true parameter, our estimation is likely to improve. However, increasing \( \sigma \), we can sample from a wider parameter space and are therefore more likely to obtain parameter samples that cover \( \theta^* \). Thus, \( \sigma \) also encodes an inherent exploration-exploitation trade-off.

\subsection{Convergence Guarantees}

As we have shown that our FSM gradient estimator closely relates to the Gaussian smoothing gradient estimator in Section~\ref{sec:gaussian-smoothing}, we can leverage established results showing the asymptotic convergence of stochastic gradient-based optimization methods with such biased gradient estimators. In particular, instead of using the final parameter iterate of the gradient-based optimization procedure as the approximated MLE \( \theta_T \approx \hat{\theta}_{\mathrm{MLE}} \), we instead propose using an averaged SGD estimator \( \bar{\theta}_T = \frac{1}{T} \sum_{t=1}^T \theta_t \) based on Polyak-Ruppert averaging \citep{Polyak_Juditsky_1992,ruppert1988efficient}, which enjoys stronger theoretical guarantees. We provide the relevant convergence arguments in Appendix~\ref{appendix:conv_guard}. 

A further benefit is that since we can obtain the quantification of the algorithmic uncertainty using the averaged SGD, \( \bar{\theta}_T - \hat{\theta}_{\mathrm{MLE}} \) from Appendix~\ref{appendix:conv_guard} and the statistical uncertainty of the MLE \( \hat{\theta}_{\mathrm{MLE}} - \theta^* \) from standard statistical theory, we can provide a result showing the quantification of the joint uncertainty using the averaged SGD \( \bar{\theta}_T \) as an approximate MLE.

\begin{theorem}
    \label{thm:uq}
    Let \( \hat{\theta}_{\mathrm{MLE},N} \) be the MLE for N i.i.d. samples. Suppose that the number of iterations in the optimization algorithm \( T \) dominates the number of observations \( N \) such that \( \sqrt{\frac{N}{T}} \to 0 \) as \( N, T \to \infty \). Then, assuming that \( \sqrt{T}(\bar{\theta}_T - \hat{\theta}_{\mathrm{MLE},N}) = \mathcal{O}_p(1) \) uniformly over both \(N, T\) and that the standard regularity conditions for the MLE are met, we have as \( N,T \to \infty\), 
    \[ \sqrt{N} (\bar{\theta}_T - \theta^*) \to_d \mathcal{N}(0, \mathcal{I}(\theta^*)^{-1} )\]
    where \( \mathcal{I}(\theta^*) \) is the Fisher information matrix evaluated at the true parameter
\end{theorem}

We provide the proof in Appendix \ref{appendix:uq_proof}. Given that the Fisher information matrix \( \mathcal{I}(\theta^*) \) can be approximated using the Fisher score by drawing samples \( \mathbf{x}_i \sim P_{\hat{\theta}} \) and evaluating \( \mathcal{I}(\theta^*) \approx \frac{1}{N} \sum_{i=1}^N \nabla_\theta \log p(\mathbf{x}_i \mid \hat{\theta}_{\mathrm{MLE}} ) \nabla_\theta \log p(\mathbf{x}_i \mid \hat{\theta}_{\mathrm{MLE}} )^\top \), we can also estimate this with our FSM method and take advantage of this result to obtain uncertainty quantification based on Theorem~\ref{thm:uq}.

\section{Related Work}

The method closest to ours is the approximate MLE approach of \citet{bertl2017approximate}, which first estimates the likelihood through kernel density estimation (KDE) before applying a simultaneous perturbation stochastic approximation (SPSA) \citep{119632} algorithm, which amounts to using a finite-differences gradient estimator on the likelihood function estimated using KDE. In contrast, our FSM method directly estimates the Fisher score, merging density and gradient estimation into one step and thereby reducing both model complexity and computational overhead. After posting the first version of this manuscript on arXiv, we became aware of related independent work by \citet{sui2025fisherscorematchingsimulationbased}, which proposes a similar Fisher score matching estimator. Their focus is on Fisher score estimation more broadly, whereas our work targets simulation-based MLE specifically.

The use of MLE in the simulation-based model setting was first addressed in the seminal work on SBI of \citet{Diggle_Gratton_1984}, although the inference of SBI is more typically addressed within the Bayesian framework, as exemplified by the ABC algorithm. Naturally, since the maximum a posteriori estimate (MAP) of the posterior distribution under a uniform prior corresponds to the MLE within the prior support, \citet{Rubio_Johansen_2013} suggested leveraging the ABC algorithm and using KDE to obtain the MLE, and more recent neural surrogate SBI methods, such as SNLE \citep{papamakarios2019sequential}, while not specifically targeted for MLE, could be used in the same way. Another similar line of research is the work of \citet{Ionides_Breto_Park_Smith_King_2017} and \citet{Park_2023}, which develop the MLE methodology in the SBI setting for partially observed Markov models. Research focused on developing SBI methods using score matching is a growing field \citep{geffner2023compositionalscoremodelingsimulationbased,sharrock2024sequentialneuralscoreestimation,jiang2025simulationbasedinferencelangevindynamics}, however, this has been limited to amortized Bayesian inference, and, to our knowledge, we are the first work that has adapted score matching for the purpose of direct Fisher score estimation and MLE in the SBI setting.

\section{Experimental Results}
\label{sec:exp_result}
We evaluate our local Fisher score matching (FSM) technique on both controlled numerical studies and challenging real-world SBI problems.\footnote{Code is available at: \url{https://github.com/Shermjj/Direct_FSM}} For all experiments, we use an isotropic Gaussian proposal distribution \( q(\theta \mid \theta_t) = \mathcal{N}(\theta_t, \sigma^2 I) \) with a linear FSM estimator, and an empirical comparison with the neural network-based FSM estimator is provided in the relevant experiment sections in the Appendix.

As a primary baseline, we compare against the approximate MLE method of \citet{bertl2017approximate}, here referred to as \emph{KDE-SP}, which estimates a log-likelihood via kernel density estimation (KDE) and then uses a simultaneous perturbation (SP) estimator to compute gradients:
\[
\hat{\nabla} \ell(\theta) = \delta \frac{\hat{\ell}(\theta^+) - \hat{\ell} (\theta^-)}{2 c}
\]
where \( \delta \) is a Rademacher random vector with i.i.d. entries, \( \theta^\pm = \theta \pm c \delta \), \( c \) is a perturbation constant, and \( \hat{\ell}(\theta; \mathbf{x}_{obs}) = \log \hat{p}(\mathbf{x}_{obs} \mid \theta)  \) is the log-likelihood estimated from the KDE by simulating data samples around the target parameter \( \theta \) and evaluating at the observations \( \mathbf{x}_{obs} \). We provide further details about the implementation in Appendix~\ref{appendix:kde_sp}. 

\subsection{Numerical Studies}
To investigate the accuracy of gradient estimation and parameter estimation, we begin with a multivariate Gaussian model that features a fixed covariance. This model has a closed-form Fisher score, allowing us to directly compare the estimated gradients from FSM and KDE-SP against the ground truth. Further details and results of this experiment are presented in Appendix~\ref{appendix:exp_7.1}.

\begin{figure}
    \centering
    \includegraphics[width=\linewidth]{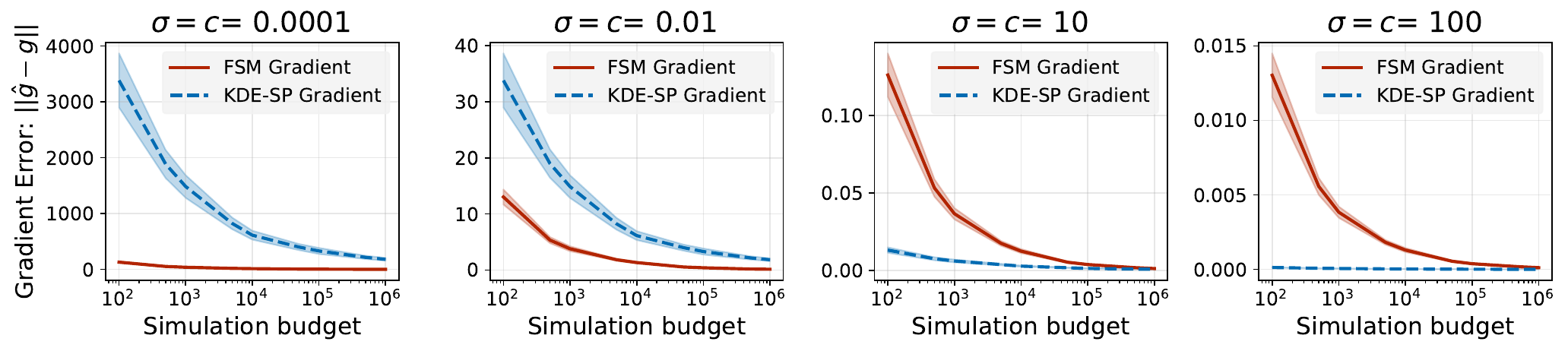}
    \caption{Comparison of FSM and KDE-SP in a 2D Gaussian model under varying hyperparameters (proposal variance or perturbation constants). Error bars show 95\% CIs over 100 repeated gradient approximations.}
    \label{fig:grad_comp_sim_mini}
\end{figure}

\begin{wrapfigure}[15]{r}{0.5\textwidth}  
    \captionsetup{skip=0.1pt}
    \centering
    \includegraphics[width=\linewidth]{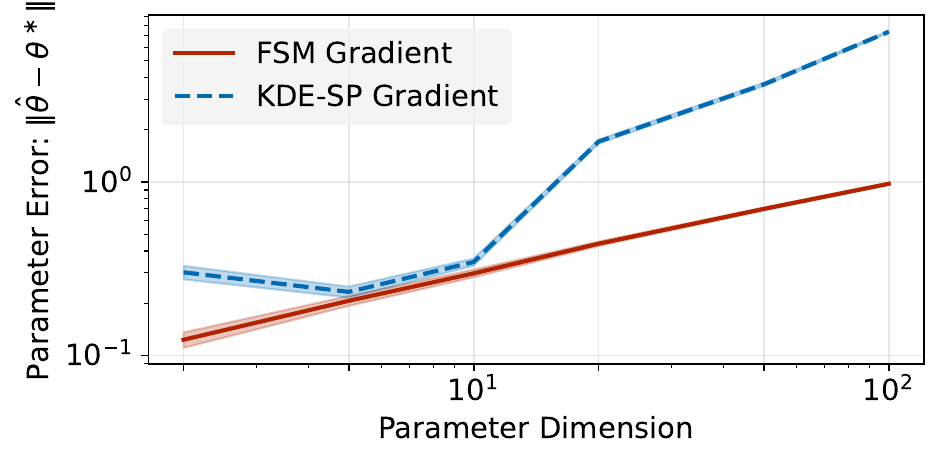}
    \caption{Parameter estimation accuracy of both the FSM and KDE-SP methods under increasing parameter dimensions, over 100 repeated optimization runs.} 
    \label{fig:mvt_gaussian_param_dim_scaling}
\end{wrapfigure}

One key aspect of both the FSM and KDE-SP approach is the choice of the hyperparameters, specifically the perturbation constant in KDE-SP and the proposal variance in FSM. In Figure~\ref{fig:grad_comp_sim_mini}, we show the sensitivity of the gradient approximation quality to different choices of this hyperparameter, as the simulation budget increases. Although the gradient approximation of both methods depends strongly on the choice of hyperparameters, we see that the FSM estimate is always able to match the accuracy of the KDE-SP estimate given sufficient simulation budget, even when the hyperparameters are not favorably tuned. We provide the same ablation study in higher-dimensional settings in Appendix~\ref{appendix:exp_7.1}.

\begin{wrapfigure}[10]{r}{0.5\textwidth}  
    \captionsetup{skip=0.1pt}
    \centering
    \includegraphics[width=\linewidth]{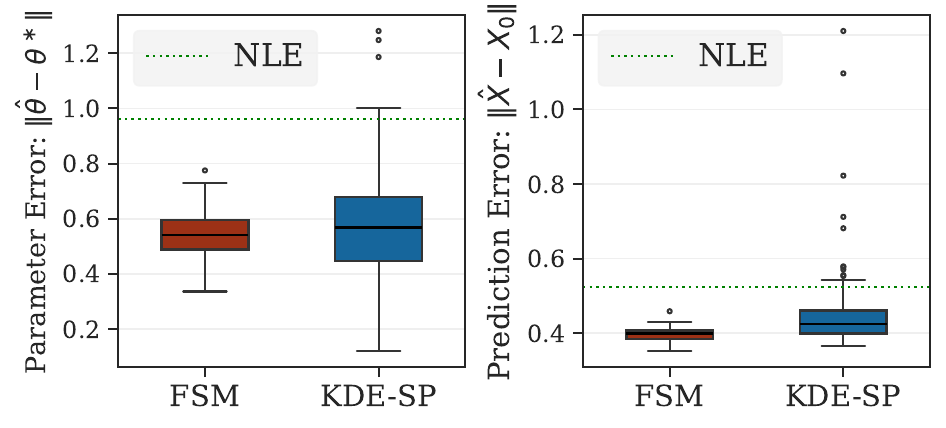}
    \caption{Parameter estimation and prediction accuracy of the NLE, FSM and KDE-SP methods.} 
    \label{fig:cosmo_boxplots_mini}
\end{wrapfigure}

In Figure~\ref{fig:mvt_gaussian_param_dim_scaling}, we show the quality of the resulting parameter estimate for the same multivariate Gaussian model with increasing parameter dimension while keeping the simulation budget fixed. While the FSM gradient is able to maintain the quality of the parameter estimate, the KDE-SP struggles in higher dimensional parameter spaces, likely due to the additional kernel density estimation required.

\subsection{LSST Weak Lensing Cosmology Model}

In this example, we use the log-normal forward model proposed by \citet{zeghal2024simulationbasedinferencebenchmarklsst} and \citet{lanzieri2025optimalneuralsummarisationfullfield}, which simulates the non-Gaussian structure in gravitational weak-lensing. Using the model in the full LSST-Y10 setting, this model is representative of real-world weak-lensing data. Since the generated data are high-dimensional tomographic convergence maps (\( 5 \times 256 \times 256 \)), we use a trained ResNet-18 compressor in \citet{alsing2018massive}, producing a 6-dimensional summary statistic. As an additional benchmark beyond the KDE-SP method, we further implement a standard neural likelihood estimator (NLE) using the \textsc{sbi} package \citep{BoeltsDeistler_sbi_2025}, trained with the same total simulation budget given to both the KDE-SP and FSM gradient-based optimization methods. Evaluated at the observations, NLE can be directly optimized to obtain an approximate maximum likelihood estimator. Further details and results of this experiment are presented in Appendix~\ref{appendix:exp_7.2}

In Figure~\ref{fig:cosmo_boxplots_mini}, we show both the parameter estimation and the accuracy of the prediction. Given the limited simulation budget available, we observe that sequential gradient-based optimization methods outperform the more simulation-intensive NLE approach and that the FSM approach is generally able to achieve better performance with a smaller variance.

\subsection{Generator Inversion Task}
\begin{wrapfigure}[16]{r}{0.5\textwidth}
    \captionsetup{skip=0.1pt}
    \centering
    \includegraphics[width=\linewidth]{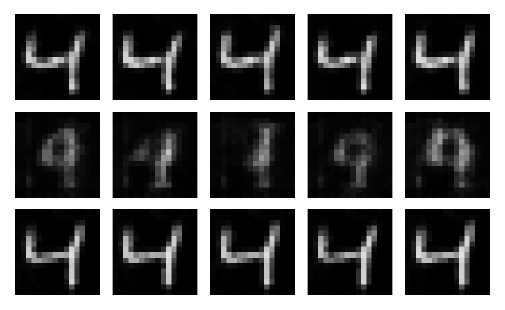}
    \caption{Images from different latent mean optimization procedure. Top row: FSM, Middle row: KDE-SP, Bottom row: Direct optimization} 
    \label{fig:mnist_half_image}
\end{wrapfigure}

In this section, we tackle the canonical problem of latent inversion of a generator network \citep{xia2022gan}. For a fixed generator \( G_w \) and a query image \( \mathbf{x}_0 \), the goal is to recover a latent vector \(\mathbf{z}\) such that \(G_w(\mathbf{z}) \approx \mathbf{x}_0 \). Although typically \( \mathbf{z} \) is treated as a point estimate, in this setting, we treat it as a latent variable, \(\mathbf{z} \sim \mathcal{N}\left(\theta, \sigma_{\mathbf{z}}^2 I\right) \) and focus on \( \theta \) as the parameter of interest. Note the marginal likelihood

\[
p_w(\mathbf{x} \mid \theta)=\int \delta\left(\mathbf{x}-G_w(\mathbf{z})\right) \mathcal{N}\left(\mathbf{z} \mid \theta, \sigma_{\mathbf{z}}^2 I\right) \mathrm{d} \mathbf{z}
\]

is intractable because the push-forward density under \(G_w\) has no closed form. However, our FSM approach allows us to estimate the Fisher score \(\nabla_\theta \log p_w(\mathbf{x} \mid \theta) \) at \(\mathbf{x}_0\), enabling us to maximize the likelihood \( \ell\left(\theta ; \mathbf{x}_0\right)=\log p_w\left(\mathbf{x}_0 \mid \theta\right) \) without directly estimating the likelihood. Conceptually, this turns the generator inversion problem into likelihood-based inference. Alternatively, given a differentiable generator, we can directly optimize the reconstruction loss to obtain an estimated latent mean (referred to as \emph{direct optimization}).

\begin{wrapfigure}[10]{r}{0.5\textwidth}  
    \captionsetup{skip=0.1pt}
    \centering
    \includegraphics[width=\linewidth]{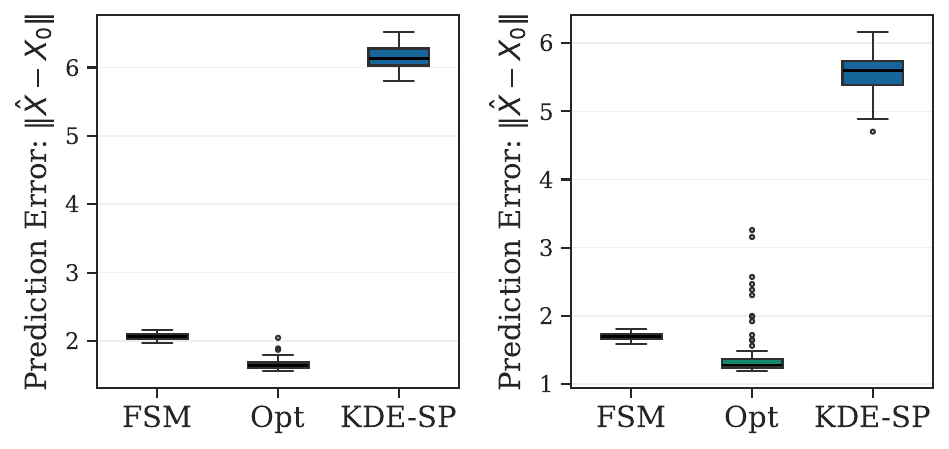} 
    \caption{Prediction error for the FSM, KDE-SP and direct optimization method.} 
    \label{fig:mnist_single}
\end{wrapfigure}

We train a GAN model on a \( 16 \times 16 \) MNIST dataset and apply the generator inversion task, comparing the direct optimization approach, FSM, and KDE-SP method. From Figure~\ref{fig:mnist_half_image} we can see that while the FSM and direct optimization is able to recover the target observation, the KDE-SP struggles to achieve the same pixel quality. This is also reflected in Figure~\ref{fig:mnist_single}, which shows the reconstruction loss for the different methods. More details and results for this experiment are provided in Appendix~\ref{appendix:exp_7.3}.

\section{Conclusion}
\label{sec:discussion}
We introduced FSM-MLE, a novel likelihood-free maximum likelihood estimation technique based on local Fisher score matching. By directly estimating the Fisher score in a simulation-based setting, our method circumvents the need to approximate the likelihood. This significantly reduces the complexity of existing approaches that either rely on kernel density estimation or train expensive neural density estimators. We further showed that under an isotropic Gaussian proposal, our local Fisher score matching estimator admits a natural Gaussian smoothing interpretation, thereby inheriting robustness properties from well-studied Gaussian smoothing techniques in black-box optimization. Empirical results on synthetic examples, a cosmological weak-lensing model, and a generator inversion task highlight the simulation efficiency and robustness of our approach. Further work includes development of a more principled selection of the proposal variance \( \sigma^2 \), a richer parameterization of the Fisher score model beyond the linear model, and further investigation into better leveraging the smoothing behavior to tackle challenging likelihood optimization in the SBI setting, as well as utilizing the approximate Fisher information matrix for uncertainty quantification. Since our method is inherently sequential, a promising extension is a "semi-amortized" variant which would leverage a pretrained neural network encoder with a training dataset, which would be coupled with our proposed linear FSM model during inference, thereby enabling a more expressive model while still preserving the benefits of fast, closed-form updates. 

\begin{ack}
        The authors thank the four anonymous reviewers for their constructive feedback. SK was supported by the EPSRC Center for Doctoral Training in Computational Statistics and Data Science, grant number EP/S023569/1.
    
\end{ack}
{ 
\small 
\bibliographystyle{plainnat} 
\bibliography{ref}    
} 

\newpage

\section*{NeurIPS Paper Checklist}

\begin{enumerate}

\item {\bf Claims}
    \item[] Question: Do the main claims made in the abstract and introduction accurately reflect the paper's contributions and scope?
    \item[] Answer: \answerYes{} 
    \item[] Justification: See Section \ref{sec:fisher_score_matching}

\item {\bf Limitations}
    \item[] Question: Does the paper discuss the limitations of the work performed by the authors?
    \item[] Answer: \answerYes{}
    \item[] Justification: The limitations are briefly discussed in Section \ref{sec:discussion} and left for future work.

\item {\bf Theory Assumptions and Proofs}
    \item[] Question: For each theoretical result, does the paper provide the full set of assumptions and a complete (and correct) proof?
    \item[] Answer: \answerYes{} 
    \item[] Justification: See Appendix \ref{appendix:conv_guard}

    \item {\bf Experimental Result Reproducibility}
    \item[] Question: Does the paper fully disclose all the information needed to reproduce the main experimental results of the paper to the extent that it affects the main claims and/or conclusions of the paper (regardless of whether the code and data are provided or not)?
    \item[] Answer: \answerYes{} 
    \item[] Justification: See Section \ref{sec:exp_result} and Algorithm \ref{alg:cap}

\item {\bf Open access to data and code}
    \item[] Question: Does the paper provide open access to the data and code, with sufficient instructions to faithfully reproduce the main experimental results, as described in supplemental material?
    \item[] Answer: \answerYes{} 
    \item[] Justification: Only publicly accessible datasets are used and code is provided. 

\item {\bf Experimental Setting/Details}
    \item[] Question: Does the paper specify all the training and test details (e.g., data splits, hyperparameters, how they were chosen, type of optimizer, etc.) necessary to understand the results?
    \item[] Answer: \answerYes{} 
    \item[] Justification: See Appendix \ref{appendix:exp_7.1}, \ref{appendix:exp_7.2},  \ref{appendix:exp_7.3}

\item {\bf Experiment Statistical Significance}
    \item[] Question: Does the paper report error bars suitably and correctly defined or other appropriate information about the statistical significance of the experiments?
    \item[] Answer: \answerYes{} 
    \item[] Justification: All figures in Section \ref{sec:exp_result} and Appendix \ref{appendix:exp_7.1}, \ref{appendix:exp_7.2}, and \ref{appendix:exp_7.3} include either error bars representing the 95\% confidence intervals over 100 repeated runs, or boxplots that visualize the distribution of the results.

\item {\bf Experiments Compute Resources}
    \item[] Question: For each experiment, does the paper provide sufficient information on the computer resources (type of compute workers, memory, time of execution) needed to reproduce the experiments?
    \item[] Answer: \answerYes{} 
    \item[] Justification: See Appendix \ref{appendix:exp_7.1}, \ref{appendix:exp_7.2}, and \ref{appendix:exp_7.3}
    
\item {\bf Code Of Ethics}
    \item[] Question: Does the research conducted in the paper conform, in every respect, with the NeurIPS Code of Ethics \url{https://neurips.cc/public/EthicsGuidelines}?
    \item[] Answer: \answerYes{} 
    \item[] Justification: The research presented in this paper fully complies with the NeurIPS Code of Ethics

\item {\bf Broader Impacts}
    \item[] Question: Does the paper discuss both potential positive societal impacts and negative societal impacts of the work performed?
    \item[] Answer: \answerNA{} 
    \item[] Justification: This paper is purely a mathematical work and does not involve direct societal applications.
    
\item {\bf Safeguards}
    \item[] Question: Does the paper describe safeguards that have been put in place for responsible release of data or models that have a high risk for misuse (e.g., pretrained language models, image generators, or scraped datasets)?
    \item[] Answer: \answerNA{} 
    \item[] Justification: This paper does not work on language models.

\item {\bf Licenses for existing assets}
    \item[] Question: Are the creators or original owners of assets (e.g., code, data, models), used in the paper, properly credited and are the license and terms of use explicitly mentioned and properly respected?
    \item[] Answer: \answerYes{} 
    \item[] Justification: See Section \ref{sec:exp_result}

\item {\bf New Assets}
    \item[] Question: Are new assets introduced in the paper well documented and is the documentation provided alongside the assets?
    \item[] Answer: \answerNA{} 
    \item[] Justification: This paper does not release new assets.

\item {\bf Crowdsourcing and Research with Human Subjects}
    \item[] Question: For crowdsourcing experiments and research with human subjects, does the paper include the full text of instructions given to participants and screenshots, if applicable, as well as details about compensation (if any)? 
    \item[] Answer: \answerNA{} 
    \item[] Justification: This paper does not involve crowdsourcing and human subjects.

\item {\bf Institutional Review Board (IRB) Approvals or Equivalent for Research with Human Subjects}
    \item[] Question: Does the paper describe potential risks incurred by study participants, whether such risks were disclosed to the subjects, and whether Institutional Review Board (IRB) approvals (or an equivalent approval/review based on the requirements of your country or institution) were obtained?
    \item[] Answer: \answerNA{} 
    \item[] Justification: This paper does not involve crowdsourcing and human subjects.

\end{enumerate}

\newpage

\appendix

\section{Appendix / Supplemental Material}
\subsection{Fisher score matching objective}
\label{appendix:lsm_proof}

Here, we provide a complete theorem and a proof for Theorem~\ref{thm:lsm}.

\begin{theorem}[Local Fisher Score Matching]
\label{thm:A1}
Let \(\mathcal{J}(W\)) be defined as in Equation~\eqref{eq:lsm_def}. Given the following assumptions:
\begin{itemize}
    \item \( p(\mathbf{x} \mid \theta)\), \( q(\theta \mid \theta_t) \) are differentiable with respect to \( \theta \),  \(S_W(\mathbf{x}) \) is differentiable with respect to \( \mathbf{x} \) 
    
    \item \(\forall \mathbf{x} \in \mathbb{R}^{d_\mathbf{x}}, \underset{\|\theta \| \to \infty}{\lim} p(\mathbf{x} \mid \theta) q(\theta \mid \theta_t) = 0\)
\end{itemize}

\( \mathcal{J}(W) \) can be rewritten (up to an additive constant w.r.t.\ \( W \)) as
\begin{align}
  \mathcal{J}(W; \theta_t) 
  \;=\;
  \mathbb{E}_{\mathbf{x}\sim p(\mathbf{x} \mid \theta), \theta \sim q(\theta \mid \theta_t)}
  [
      \|S_{W}(\mathbf{x}) \|^{2}
      \;+\;
      2\,S_{W}(\mathbf{x})^{\top}
      \,\nabla_{\theta}\log q(\theta \mid \theta_t)
    ]
\end{align}
\end{theorem}

\begin{proof}
We denote the joint distribution over \( (\mathbf{x}, \theta) \) from the distributions \( p(\mathbf{x} \mid \theta) \) and \( q(\theta \mid \theta_t) \) as \( p(\mathbf{x}, \theta \mid \theta_t) \). First, we expand the square, and remove terms which are not dependent on the score model parameters \( W \). 

\begin{flalign*}
\mathcal{J}(W) &= \mathbb{E}_{p(\mathbf{x}, \theta \mid \theta_t)}(\| \nabla_\theta \log p(\mathbf{x} \mid \theta) - S_W(\mathbf{x}) \|^2) && \\
&= \int_{\theta \in \mathbb{R}^{d_\theta}}q(\theta \mid \theta_t) \int_{\mathbf{x} \in \mathbb{R}^{d_\mathbf{x}}} p(\mathbf{x} \mid \theta) \| \nabla_\theta \log p(\mathbf{x} \mid \theta) - S_W(\mathbf{x}) \|^2 d\mathbf{x} \; d\theta && \\
&= \int_{\theta \in \mathbb{R}^{d_\theta}}q(\theta \mid \theta_t) \int_{\mathbf{x} \in \mathbb{R}^{d_\mathbf{x}}} p(\mathbf{x} \mid \theta) \{ \| \nabla_\theta \log p(\mathbf{x} \mid \theta) \|^2 + \| S_W(\mathbf{x}) \|^2 && \\
&\qquad - 2 \nabla_\theta \log p(\mathbf{x} \mid \theta)^{\top} S_W(\mathbf{x}) \} d\mathbf{x} \; d\theta && \\
&= \int_{\theta \in \mathbb{R}^{d_\theta}}q(\theta \mid \theta_t) \int_{\mathbf{x} \in \mathbb{R}^{d_\mathbf{x}}} p(\mathbf{x} \mid \theta) \left\{ \| S_W(\mathbf{x}) \|^2 - 2 \nabla_\theta \log p(\mathbf{x} \mid \theta)^\top S_W(\mathbf{x}) \right\} d\mathbf{x} \; d\theta && \\
&\qquad + (\text{constants w.r.t. } W ) &&
\end{flalign*}

Next, by exchanging integrals and using the integration by parts tricks similar to Theorem 1 in \citet{hyvarinen2005estimation},

\begin{flalign*}
\mathcal{J}(W) &= \int_{\theta \in \mathbb{R}^{d_\theta}}q(\theta \mid \theta_t) \int_{\mathbf{x} \in \mathbb{R}^{d_\mathbf{x}}} p(\mathbf{x} \mid \theta) \| S_W(\mathbf{x}) \|^2 d\mathbf{x} \; d\theta && \\
&\qquad - 2 \int_{\theta \in \mathbb{R}^{d_\theta}}q(\theta \mid \theta_t) \int_{\mathbf{x} \in \mathbb{R}^{d_\mathbf{x}}} p(\mathbf{x} \mid \theta) \nabla_\theta \log p(\mathbf{x} \mid \theta)^\top S_W(\mathbf{x}) d\mathbf{x} \; d\theta && \\
&= \int_{\theta \in \mathbb{R}^{d_\theta}}q(\theta \mid \theta_t) \int_{\mathbf{x} \in \mathbb{R}^{d_\mathbf{x}}} p(\mathbf{x} \mid \theta) \| S_W(\mathbf{x}) \|^2 d\mathbf{x} \; d\theta && \\
&\qquad - 2 \int_{\mathbf{x} \in \mathbb{R}^{d_\mathbf{x}}} \int_{\theta \in \mathbb{R}^{d_\theta}} q(\theta \mid \theta_t) \nabla_\theta p(\mathbf{x} \mid \theta)^\top S_W(\mathbf{x}) d\theta \; d\mathbf{x} && \\
&= \int_{\theta \in \mathbb{R}^{d_\theta}}q(\theta \mid \theta_t) \int_{\mathbf{x} \in \mathbb{R}^{d_\mathbf{x}}} p(\mathbf{x} \mid \theta) \| S_W(\mathbf{x}) \|^2 d\mathbf{x} \; d\theta && \\
&\qquad - 2 \int_{\mathbf{x} \in \mathbb{R}^{d_\mathbf{x}}} \int_{\theta \in \mathbb{R}^{d_\theta}} q(\theta \mid \theta_t) \sum\limits_{i=1}^{d_
\theta} \frac{\partial}{\partial \theta_i} p(\mathbf{x} \mid \theta) S^{(i)}_W(\mathbf{x}) d\theta \; d\mathbf{x} &&
\end{flalign*}

\begin{flalign*}
\mathcal{J}(W) &= \int_{\theta \in \mathbb{R}^{d_\theta}}q(\theta \mid \theta_t) \int_{\mathbf{x} \in \mathbb{R}^{d_\mathbf{x}}} p(\mathbf{x} \mid \theta) \| S_W(\mathbf{x}) \|^2 d\mathbf{x} \; d\theta && \\
&\qquad - 2 \int_{\mathbf{x} \in \mathbb{R}^{d_\mathbf{x}}} \sum\limits_{i=1}^{d_\theta} S^{(i)}_W(\mathbf{x}) \int_{\theta \in \mathbb{R}^{d_\theta}} q(\theta \mid \theta_t) \frac{\partial}{\partial \theta_i} p(\mathbf{x} \mid \theta)  d\theta \; d\mathbf{x} && \\
&= \int_{\theta \in \mathbb{R}^{d_\theta}}q(\theta \mid \theta_t) \int_{\mathbf{x} \in \mathbb{R}^{d_\mathbf{x}}} p(\mathbf{x} \mid \theta) \| S_W(\mathbf{x}) \|^2 d\mathbf{x} \; d\theta && \\
&\qquad + 2 \int_{\mathbf{x} \in \mathbb{R}^{d_\mathbf{x}}} \sum\limits_{i=1}^{d_\theta} S^{(i)}_W(\mathbf{x}) \int_{\theta \in \mathbb{R}^{d_\theta}} \frac{\partial}{\partial \theta_i} q(\theta \mid \theta_t)  p(\mathbf{x} \mid \theta)  d\theta \; d\mathbf{x} &&
\end{flalign*}

Finally, by further simplification
\begin{flalign*}
\mathcal{J}(W) &= \int_{\theta \in \mathbb{R}^{d_\theta}}q(\theta \mid \theta_t) \int_{\mathbf{x} \in \mathbb{R}^{d_\mathbf{x}}} p(\mathbf{x} \mid \theta) \| S_W(\mathbf{x}) \|^2 d\mathbf{x} \; d\theta && \\
&\qquad + 2 \int_{\mathbf{x} \in \mathbb{R}^{d_\mathbf{x}}} \sum\limits_{i=1}^{d_\theta} S^{(i)}_W(\mathbf{x}) \int_{\theta \in \mathbb{R}^{d_\theta}} q(\theta \mid \theta_t) \frac{\partial}{\partial \theta_i}\log q(\theta \mid \theta_t)  p(\mathbf{x} \mid \theta)  d\theta \; d\mathbf{x} && \\
&= \int_{\theta \in \mathbb{R}^{d_\theta}}q(\theta \mid \theta_t) \int_{\mathbf{x} \in \mathbb{R}^{d_\mathbf{x}}} p(\mathbf{x} \mid \theta) \| S_W(\mathbf{x}) \|^2 d\mathbf{x} \; d\theta && \\
&\qquad + 2 \int_{\theta \in \mathbb{R}^{d_\theta}} q(\theta \mid \theta_t) \int_{\mathbf{x} \in \mathbb{R}^{d_\mathbf{x}}} \sum\limits_{i=1}^{d_\theta} S^{(i)}_W(\mathbf{x})  \frac{\partial}{\partial \theta_i}\log q(\theta \mid \theta_t)  p(\mathbf{x} \mid \theta)  d\mathbf{x} \; d\theta && \\
&= \int_{\theta \in \mathbb{R}^{d_\theta}} q(\theta \mid \theta_t) \int_{\mathbf{x} \in \mathbb{R}^{d_\mathbf{x}}} p(\mathbf{x} \mid \theta) \sum\limits_{i=1}^{d_\theta} \left[ S^{(i)}_W(\mathbf{x})^2 + 2 S^{(i)}_W(\mathbf{x})  \frac{\partial}{\partial \theta_i}\log q(\theta \mid \theta_t) \right]   d\mathbf{x} \; d\theta && \\
&= \underset{q(\theta \mid \theta_t)}{\mathbb{E}}  \underset{p(\mathbf{x} \mid \theta)}{\mathbb{E}}\sum\limits_{i=1}^{d_\theta} \left[ S^{(i)}_W(\mathbf{x})^2 + 2 S^{(i)}_W(\mathbf{x})  \frac{\partial}{\partial \theta_i}\log q(\theta \mid \theta_t) \right]  && \\
&= \mathbb{E}_{\mathbf{x}\sim p(\mathbf{x} \mid \theta), \theta \sim q(\theta \mid \theta_t)}
  [
      \|S_{W}(\mathbf{x}) \|^{2}
      \;+\;
      2\,S_{W}(\mathbf{x})^{\top}
      \,\nabla_{\theta}\log q(\theta \mid \theta_t)
    ] &&
\end{flalign*}

\end{proof}

\subsection{Bayes-optimal solution to Fisher score matching objective }
\label{appendix:bo-proof}

We present the complete theorem and proof for Theorem~\ref{thm:bo-fsm} here.
\begin{theorem}
    For a general differentiable function \(S: \mathbb{R}^{d_\mathbf{x}} \to \mathbb{R}^{d_\theta} \), 
\[
    S^* = \underset{S}{\mathrm{argmin}} \underset{p(\mathbf{x},\theta \mid \theta_t)}{\mathbb{E}} \| \nabla_\theta \log p(\mathbf{x} \mid \theta) - S(\mathbf{x}) \|^2 = \underset{p(\theta \mid \mathbf{x}, \theta_t)}{\mathbb{E}} \nabla_\theta \log p(\mathbf{x} \mid \theta) 
\]

\end{theorem}
\begin{proof}
First, observe that since the function \(S\) is only a function of \(\mathbf{x}\), we have
\[
\underset{p(\mathbf{x}, \theta \mid \theta_t)}{\mathbb{E}} S(\mathbf{x}) = \underset{p(\mathbf{x}\mid \theta_t)}{\mathbb{E}} S(\mathbf{x})
\]

We can decompose the objective function by expanding the square,
\[
\underset{S}{\mathrm{argmin}}  \underset{p(\mathbf{x}, \theta \mid \theta_t)}{\mathbb{E}}\| \nabla_\theta \log p(\mathbf{x} \mid \theta) - S(\mathbf{x}) \|^2  = \underset{S}{\mathrm{argmin}}  \underset{p(\mathbf{x}, \theta \mid \theta_t)}{\mathbb{E}}\left(\|S(\mathbf{x})\|^2 - 2 S(\mathbf{x})^\top \nabla_\theta \log p(\mathbf{x} \mid \theta)\right)
\]

Then, our objective can be equivalently expressed as 
\[
\underset{p(\mathbf{x} \mid \theta_t)}{\mathbb{E}}\left[\|S(\mathbf{x})\|^2 - 2 S(\mathbf{x})^\top 
\underset{p(\theta\mid\mathbf{x},\theta_t)}{\mathbb{E}}[\nabla_\theta \log p(\mathbf{x} \mid \theta)]\right]
\]

Which has the optimal solution \(S^*(\mathbf{x}) = \underset{p(\theta \mid \mathbf{x}, \theta_t)}{\mathbb{E}} \nabla_\theta \log p(\mathbf{x} \mid \theta)\)
\end{proof}

\subsection{Linear Fisher score model parameterization}
\label{appendix:linear_model}
Here, we provide details of the linear Fisher score model derivation.

Recall that the parameter and data space are \( \theta \in \mathbb{R}^{d_\theta} \), \( \mathbf{x}^{(j)}_k \in \mathbb{R}^{d_\mathbf{x}} \), the linear score model weights are \( W \in \mathbb{R}^{d_\mathbf{x} \times d_\theta} \), and we defined the data matrix as \( X_j = \begin{bmatrix} 
                                                \mathbf{x}_{1}^{(j) \top} \\
                                                \vdots \\
                                                \mathbf{x}_{n}^{(j) \top}
                                                \end{bmatrix} \in \mathbb{R}^{n \times d_\mathbf{x}} \) and the corresponding Gram matrix as \( G_j = X_j^\top X_j \).

In practice, we include an intercept term in our regression by augmenting the data matrix with a column of ones, i.e., \( \begin{bmatrix}
    \mathbf{x}_{k}^{(j)} \\
    1
\end{bmatrix} \in \mathbb{R}^{d_\mathbf{x} +1} \) and \( W \) as a \( (d_\mathbf{x}+1) \times d_\theta \) matrix. For simplicity, we omit this intercept term in our derivation.

We start from the empirical version of the local Fisher score matching objective, Equation~\eqref{eq:lsm_est} (replacing averages by sums for simplicity),

\[
\hat{\mathcal{J}}(W) = \sum\limits_{j=1}^{m} \sum\limits_{k=1}^{n}   [
      \|S_{W}(\mathbf{x}^{(j)}_k) \|^{2}
      \;+\;
      2\,S_{W}(\mathbf{x}^{(j)}_k)^{\top}
      \, \nabla_{\theta}\log q(\theta \mid \theta_t)|_{\theta = \theta^{(j)}}
    ]
\]

Substituting our linear score model, \( S(\mathbf{x}; \theta_t) = W^\top \mathbf{x} \),
\[
    \hat{\mathcal{J}}(W) = \sum\limits_{j=1}^{m} \sum\limits_{k=1}^{n} \| W^\top \mathbf{x}^{(j)}_k \|^2 + 2 \sum\limits_{j=1}^{m} \sum\limits_{k=1}^{n}  ( \mathbf{x}_k^{(j) \top} W \nabla_{\theta}\log q(\theta \mid \theta_t)|_{\theta = \theta^{(j)}} )
\]

To obtain the first-order conditions, we take derivative with respect to \( W \), for each of the terms separately. 

For the first term,
\begin{align*}
    \sum\limits_{j=1}^{m} \sum\limits_{k=1}^{n} \| W^\top \mathbf{x}^{(j)}_k \|^2 &=  \sum\limits_{j=1}^{m} \mathrm{tr}[(X_j W)^\top (X_j W)] \\   
    &= \sum\limits_{j=1}^{m} \mathrm{tr}(W^\top X_j^{\top} X_j W)
\end{align*}

Applying \( \frac{\partial}{\partial W}\) gives:
\[  \sum\limits_{j=1}^{m} 2 X_j^{\top} X_j W = 2 \sum\limits_{j=1}^{m} G_j W\]

For the second term, we can similarly apply \( \frac{\partial}{\partial W}\) to give:
\[ 
2 \sum\limits_{j=1}^{m} \sum\limits_{k=1}^{n}  \frac{\partial}{\partial W} [\mathbf{x}_k^{(j)\top} W  \nabla_{\theta}\log q(\theta \mid \theta_t)|_{\theta = \theta^{(j)}}]  = 2 \sum\limits_{j=1}^{m} \sum\limits_{k=1}^{n} \mathbf{x}^{(j)}_k \nabla_{\theta}\log q(\theta \mid \theta_t)|_{\theta = \theta^{(j)}}^\top 
\]

Combining the two terms, we obtain

\[ 
\frac{\partial}{\partial W} \hat{\mathcal{J}}(W) = 2\sum\limits_{j=1}^{m} G_j W + 2\sum\limits_{j=1}^{m} \sum\limits_{k=1}^{n} \mathbf{x}^{(j)}_k \nabla_{\theta}\log q(\theta \mid \theta_t)|_{\theta = \theta^{(j)}}^\top 
\]

Setting this to 0 gives us the normal equations in Equation~\eqref{eq:normal_lsm}.

If the sum of the Gram matrices, \( \sum\limits_{j=1}^{m} G_j \) is invertible (otherwise, we may opt to use the ridge penalty), we can directly obtain the linear Fisher score matching estimator in Equation~\eqref{eq:linear_fish}.

Naturally, our linear score model setup can be extended to include a Frobenius norm penalty \(\lambda \|W\|^{2}_{F}\) in the objective, leading to a ridge-type solution:
\[
\hat{W}
\;=\;
-\,\bigl[\sum_{j=1}^{m} G_j + \lambda \, I_{d_\mathbf{x}} \bigr]^{-1}
\sum_{j=1}^{m}\sum_{k=1}^n
\left[ \mathbf{x}^{(j)}_k  \nabla_{\theta}\log q(\theta \mid \theta_t)|_{\theta = \theta^{(j)}}^{\top} \right]
\]

This stabilises the inverse and helps prevent overfitting in finite-sample regimes. In practice, we implement the ridge-type linear Fisher score model. 

\subsection{Neural Network Fisher score model parameterization}
\label{appendix:nn_model}
An alternative to the linear Fisher score model provided in Appendix~\ref{appendix:linear_model} is a neural network parameterization of the score model. In this setting, we denote \( S(\mathbf{x}; \theta_t) = S_{\phi}(\mathbf{x}) \) where \( S_{\phi} \) is a neural network with parameters \( \phi \) for a fixed parameter iterate \( \theta_t \). The parameters \( \phi \) can be obtained by optimizing the FSM objective \( \mathcal{J}(\phi; \theta_t) \) from Equation~\eqref{eq:lsm_simplified} (with its Monte Carlo estimate Equation~\eqref{eq:lsm_est}) using standard neural network backpropagation. Thus, following Algorithm~\ref{alg:cap}, using a neural network parameterization requires a potentially costly inner optimization loop for each parameter iterate \( \theta_t \).    

\subsection{Fisher score matching proposal distribution}
\label{appendix:fsm_proposal}

As discussed in Section~\ref{subsec:fsm_properties}, the theoretical optimal choice of the hyperparameter \( \sigma \) depends on the curvature of the log-likelihood function. In practice, the curvature is unknown, and thus selecting the optimal \( \sigma \) is challenging in general. We propose a simple pilot calibration based on grid search, before running the main FSM-MLE procedure, we execute a short pilot FSM-MLE procedure with different candidate value \( \sigma_k \) yielding corresponding candidate parameter estimate \(\hat{\theta}_{\sigma_k}\). For each \( \sigma_k \), we simulate data at \(\hat{\theta}_{\sigma_k}\), and select the candidate \( \sigma_k \) which minimizes the discrepancy between the observed data and the simulated data at the candidate parameter estimate. Thus, this procedure selects hyperparameters \(\sigma\) which can produce simulations most consistent with the observations. 

A simple annealing schedule which would reduce the hyperparameter \( \sigma \) over the course of the FSM-MLE optimization procedure could also be considered. This would ensure that the smoothing bias discussed in Section~\ref{sec:theory} vanishes asymptotically, however, such a procedure would be complicated by the increase in the variance of the FSM estimator and numerical instability when \( \sigma \) is too small. While we attempted to implement such an annealing scheme in our experiments, we found that it introduced additional complexity without meaningful performance gains over a fixed \( \sigma \) scheme.

When the likelihood function exhibits strong anisotropic curvature, an isotropic Gaussian proposal is suboptimal. Extending the calibration scheme to diagonal covariances, however, would make grid search scale exponentially with the parameter dimension, making the method computationally prohibitive for higher dimensional problems. Hence, it remains an open question on how to efficiently design scalable procedure to select more expressive proposal distributions.

\subsection{Gaussian smoothing equivalence}
\label{appendix:gs_proof}

We provide here a more detailed derivation of Theorem~\ref{thm:gs}.

\begin{theorem}[Equivalence as Gaussian Smoothing]
    Under an isotropic Gaussian proposal, \(q(\theta \mid \theta_t) = \mathcal{N}(\theta \mid \theta_t, \sigma^2 I)\), with the assumptions as Theorem~\ref{thm:A1}, the optimal score matching estimator is equivalent to the gradient of the smoothed likelihood
    \[
    \nabla_{\theta_t} \tilde{\ell}(\theta_t ; \mathbf{x}) = \mathbb{E}_{\theta \sim p(\theta \mid \mathbf{x}, \theta_t)} \nabla_\theta \log p(\mathbf{x} \mid \theta)
    \]

    where \(\tilde{\ell}(\theta_t;\mathbf{x})
  \;=\;
  \log \int 
  p\bigl(\mathbf{x} \mid \theta\bigr)\,q\bigl(\theta \mid \theta_t\bigr)
  \,d\theta\) and \( p(\theta \mid \mathbf{x}, \theta_t) \propto p(\mathbf{x} \mid \theta) q( \theta \mid \theta_t) \) is the induced posterior from the proposal distribution \( q(\theta \mid \theta_t) \)
  
\end{theorem}

\begin{proof}

First, define \( Z(\theta_t) = \int p(\mathbf{x} \mid \theta) q(\theta \mid \theta_t) d \theta \) such that \( \tilde{\ell}(\theta_t; \mathbf{x}) = \log Z(\theta_t) \).

Now, observe that,
\[
   \nabla_{\theta_t}\,\widetilde{\ell}(\theta_t;\mathbf{x})
   \;=\;
   \frac{\nabla_{\theta_t}\,Z(\theta_t)}{Z(\theta_t)}\,
   \;=\;
   \frac{1}{Z(\theta_t)}
   \int p(\mathbf{x} \mid \theta) \nabla_{\theta_t} q(\theta \mid \theta_t) d\theta
\]

For an isotropic Gaussian proposal, \(q(\theta \mid \theta_t) = \mathcal{N}(\theta \mid \theta_t,\sigma^2I) \), we have that 

\[
\nabla_{\theta_t} q(\theta \mid \theta_t) = -\,\nabla_{\theta} q(\theta \mid \theta_t)
\]

Using the integration-by-parts trick (similarly to the proof in Appendix~\ref{appendix:gs_proof}), we have,

\begin{align*}
    \int p(\mathbf{x} \mid \theta) \nabla_{\theta_t} q(\theta \mid \theta_t) d \theta &= - \int p( \mathbf{x} \mid \theta) \nabla_\theta q(\theta \mid \theta_t) d \theta \\
    &= \int \nabla_\theta p(\mathbf{x} \mid \theta) q(\theta \mid \theta_t) d\theta \\
    &= \int \nabla_\theta \log p(\mathbf{x} \mid \theta) p(\mathbf{x} \mid \theta) q(\theta \mid \theta_t) d\theta
\end{align*}

Substituting this expression into \( \nabla_{\theta_t}\,\widetilde{\ell}(\theta_t;\mathbf{x}) \), we have,

\begin{align*}
    \nabla_{\theta_t}\,\widetilde{\ell}(\theta_t;\mathbf{x}) &= \frac{1}{Z(\theta_t)}
   \int p(\mathbf{x} \mid \theta) \nabla_{\theta_t} q(\theta \mid \theta_t) d\theta \\
   &= \frac{1}{Z(\theta_t)}
   \int \nabla_\theta \log p(\mathbf{x} \mid \theta) p(\mathbf{x} \mid \theta) q(\theta \mid \theta_t) d\theta \\
   &= \int \nabla_\theta \log p(\mathbf{x} \mid \theta) \frac{p(\mathbf{x} \mid \theta) q(\theta \mid \theta_t)}{Z(\theta_t)} d\theta \\
   &= \mathbb{E}_{\theta \sim p(\theta \mid \mathbf{x}, \theta_t)} \nabla_\theta \log p(\mathbf{x} \mid \theta)
\end{align*}
\end{proof}

\subsection{Bias of FSM}
\label{appendix:FSM-BV}
\begin{theorem}[Bias characterization of the FSM estimator]
\label{thm:bias_variance-app}
Let \( \theta^* \) be the true parameter, and denote \( \mathbf{x}_0 \sim P_{\theta^*} \) as random observations sampled from the true model. Suppose there exists a unique maximum likelihood estimator for this model, and that the log-likelihood is \(L\)-smooth. Recall that \( g(\mathbf{x}_0; \theta_t) = \nabla_\theta \log p(\mathbf{x}_0 \mid \theta)|_{\theta = \theta_t} \) is the true Fisher score, \( S^*(\mathbf{x}_0; \theta_t) = \mathbb{E}_{\theta \sim p(\theta \mid \mathbf{x}, \theta_t)} \nabla_\theta \log p(\mathbf{x} \mid \theta) \) is the optimal FSM estimator. For a fixed parameter point \( \theta_t \),

  The bias at \(\theta_t\) is bounded by
  \[
    \mathbb{E}_{\mathbf{x}_0}\!\bigl\|S^{*}(\mathbf{x}_0;\theta_t)-g(x_0;\theta_t)\bigr\|
    \;\le\;
    L\; \sqrt{d}\,\sigma\, \mathbb{E}_{\mathbf{x}_0} [R(\mathbf{x}_0)]
  \] 
      
where \( R(\mathbf{x}) = \frac{p(\mathbf{x} \mid \theta^*)}{ p(\mathbf{x} \mid \theta_t)} \) is a likelihood ratio term and \( d \) is the dimension of the parameter space 

\end{theorem}

\begin{proof}

Recall that \( g(\mathbf{x};\theta_t) = \nabla_\theta \log p(\mathbf{x} \mid \theta)|_{\theta = \theta_t}\) is the true Fisher score and the optimal FSM estimator is defined as \( S^*(\mathbf{x}; \theta_t) = \mathbb{E}_{\theta \sim p(\theta \mid \mathbf{x}, \theta_t)} \nabla_\theta \log p(\mathbf{x} \mid \theta)\).

1. We first show the bias bound \( \| S^*(\mathbf{x}; \theta_t) - g(\mathbf{x};\theta_t) \| \), at a fixed data point \( \mathbf{x} \).

\begin{align*}
    \| S^*(\mathbf{x}; \theta_t) - g(\mathbf{x};\theta_t) \| &= \Bigl \| \mathbb{E}_{\theta \sim p(\theta \mid \mathbf{x}, \theta_t)} \bigl[ \nabla_\theta \log p(\mathbf{x} \mid \theta) - \nabla_\theta \log p(\mathbf{x} \mid \theta)|_{\theta = \theta_t} \bigr] \Bigr\| \\
    &\leq \mathbb{E}_{\theta \sim p(\theta \mid \mathbf{x}, \theta_t)} \Bigl[ \bigl \|\nabla_\theta \log p(\mathbf{x} \mid \theta) - \nabla_\theta \log p(\mathbf{x} \mid \theta)|_{\theta = \theta_t} \bigr \| \Bigr]\\
    &\leq L \int \| \theta - \theta_t \| p( \theta \mid \mathbf{x} ,\theta_t) d \theta  \\
    &\leq L \sup_\theta \frac{p(\theta \mid \mathbf{x}, \theta_t)}{q(\theta \mid \theta_t)}    \int \| \theta - \theta_t \| q( \theta \mid \theta_t)  d\theta  \\
    &\leq L \sqrt{d} \sigma \sup_\theta \frac{p(\theta \mid \mathbf{x}, \theta_t)}{q(\theta \mid \theta_t)}  
\end{align*}

Denoting \( p(\mathbf{x}\mid\theta_t) := \int p(\mathbf{x}\mid\theta) \, q(\theta\mid\theta_t) \, d\theta \), note that \( \sup_\theta \frac{p(\theta \mid \mathbf{x}, \theta_t)}{q(\theta \mid \theta_t)} = \frac{p(\mathbf{x} \mid \hat{\theta}_{\mathrm{MLE}}(\mathbf{x}))}{p(\mathbf{x} \mid \theta_t)} \)

2. We now take expectation with respect to the true model, \( \mathbf{x}_0 \sim p(\mathbf{x} \mid \theta^*) \), and the only non-constant term is \( \sup_\theta \frac{p(\theta \mid \mathbf{x}, \theta_t)}{q(\theta \mid \theta_t)} \).

\begin{align*}
    \mathbb{E}_{\mathbf{x}_0} \sup_\theta \frac{p(\theta \mid \mathbf{x}_0, \theta_t)}{q(\theta \mid \theta_t)} &= \int \frac{p(\mathbf{x}_0 \mid \hat{\theta}_{\mathrm{MLE}}(\mathbf{x}_0))}{p(\mathbf{x}_0 \mid \theta_t)}\ p(\mathbf{x}_0 \mid \theta^*) d\mathbf{x}_0 \\
    &\approx \int \frac{p(\mathbf{x}_0 \mid \theta^*)}{p(\mathbf{x}_0 \mid \theta_t)}\ p(\mathbf{x}_0 \mid \theta^*) d\mathbf{x}_0  \\
    &= \mathbb{E}_{\mathbf{x}_0} \frac{p(\mathbf{x}_0 \mid \theta^*)}{p(\mathbf{x}_0 \mid \theta_t)}
\end{align*}

\end{proof}

\subsection{Convergence guarantees of FSM}
\label{appendix:conv_guard}

In this section, we provide an asymptotic convergence analysis of the stochastic gradient method based on the local Fisher score matching gradient under a Gaussian proposal distribution. Recall that, based on theoretical development in Section~\ref{sec:gaussian-smoothing} and Appendix~\ref{appendix:gs_proof}, we have shown that the FSM estimator, under a isotropic Gaussian proposal distribution, targets a smoothed log-likelihood,
\[
\tilde{\ell}_\sigma(\theta;\mathbf{x}) = \log \int p(\mathbf{x} \mid \theta') \mathcal{N}(\theta' \mid \theta,\sigma^2I)  d \theta'
\]

Let the \( N \) independent and identically distributed observations be \( \mathcal{D} = \{\mathbf{x}_i\}_{i=1}^N \) and the corresponding smoothed likelihood objective be
\[
\tilde{\ell}_\sigma(\theta; \mathcal{D}) = \sum_{i=1}^N \tilde{\ell}_\sigma(\theta;\mathbf{x}_i)
\]

We define the smoothed maximum likelihood estimator for the dataset \( \mathcal{D} \) as
\[
\hat{\theta}_\sigma = \mathrm{argmax}_\theta \sum_{i=1}^N \tilde{\ell}_\sigma(\theta;\mathbf{x}_i)
\]

Equivalently, assuming the concavity of the smoothed likelihood function, we can characterize the smoothed maximum likelihood estimator with its first-order optimality condition.

\[
\nabla_\theta \tilde{\ell}_\sigma(\theta;\mathcal{D}) = \sum_{i=1}^N S^*(\mathbf{x}_i; \theta) = 0
\]

where \( S^*(\mathbf{x}; \theta) = \nabla_\theta \tilde{\ell}_\sigma(\theta; \mathbf{x}) \) is the Bayes-optimal FSM estimator.

In practice, however, we utilize the linear FSM estimator \( \hat{S}(\mathbf{x}; \theta) \) as discussed in Section~\ref{subsec:score_model_param} and Appendix~\ref{appendix:linear_model}. In order to reduce the variance of the resulting approximate maximum likelihood estimator, as well as to provide stronger theoretical guarantees, we use the averaged parameter estimate \citep{Polyak_Juditsky_1992}

\[
\bar{\theta}_T = \frac{1}{T}\sum_{t=1}^T \theta_t
\]

In the following, we state an asymptotic convergence result, which can be found in Proposition 2.1 of \citet{jin2021statisticalinferencepolyakruppertaveraged}, which is based on \citet{Polyak_Juditsky_1992}. 

\begin{assumption}[Smoothness and concavity of the true log-likelihood] \label{ass:smooth_con}
The log-likelihood function \( \ell(\theta ; \mathcal{D}) = \sum_{i=1}^N \log p(\mathbf{x}_i \mid \theta) \) is \(L\) smooth and \(\mu\) strongly concave.
\end{assumption}

Note that this is a sufficient condition for the strong concavity of the smoothed log-likelihood.

\begin{assumption}[Step Size Condition]
\label{ass:step_size}
    The step-sizes \( \eta_t > 0\) satisfies for all \( t \), \( \frac{\eta_t - \eta_{t+1}}{\eta_t} = o( \eta_t) \) and \( \sum_{t=1}^\infty \eta^{(1+\lambda)/2} t^{-1/2} < \infty \)
\end{assumption}

\begin{assumption}[Unbiasedness and Martingale Noise Control]
\label{ass:mart_con}
    Define the noise term \( \xi_t = \hat{S}(\theta_{t-1}, u_t; \sigma) - S^*(\theta_{t-1}; \sigma) \), which is a martingale difference sequence with respect to \( \mathcal{F}_{t-1} = \sigma(u_1, \ldots,u_{t-1}) \), where \( u_t = \{ (\theta_{i,t}, X_{i,t}) \}_{i=1}^M \) represents all the simulations used for the score model estimation at iteration \( t \).
    
    \begin{enumerate}
        \item For all iterations \( t \geq 1 \), the linear FSM estimator is unbiased:
        \[
        \mathbb{E}[\hat{S}(\theta_{t-1}, u_t; \sigma) \mid \mathcal{F}_{t-1}] = S^*(\theta_{t-1}; \sigma)
        \]
        \item Assume that there exists a constant \( K > 0 \) such that for all \( t \geq 1 \), almost surely:
    \[
    \mathbb{E}[\| \xi_t \|^2 \mid \mathcal{F}_{t-1} ] + \| S^*(\theta_{t-1}; \sigma )\|^2 \leq K ( 1 + \|\theta_{t-1} - \hat{\theta}_\sigma \|^2 )
    \]
    \end{enumerate}
\end{assumption}

\begin{assumption}[Hessian Bound]
\label{ass:hess_bound}
    There is a function \( H(u) \) with bounded fourth moments, such that the operator norm of \( \nabla_\theta \hat{S}(\theta,u) \) is bounded, \( \|\nabla_\theta \hat{S}(\theta,u)\| \leq H(u) \) for all \( \theta \)
\end{assumption}

\begin{theorem}
    \label{thm:asy_cov}
    Suppose Assumptions \ref{ass:smooth_con}, \ref{ass:mart_con} and \ref{ass:hess_bound} hold and the sequence of step sizes fulfills \ref{ass:step_size}. Using the updates of the gradient descent \( \theta_{t+1} \leftarrow \theta_t + \eta_t \hat{S}_t(\mathbf{x}; \theta_t, \sigma) \), we have that the averaged parameter iterates \( \bar{\theta}_T = \frac{1}{T} \sum_{t=1}^T \theta_t\) satisfies as \( T \to \infty \):
    \begin{enumerate}
        \item \( \bar{\theta}_T \to_{a.s.} \hat{\theta}_\sigma \) 
        \item \( \sqrt{T} (\bar{\theta}_T - \hat{\theta}_\sigma) \to_d \mathcal{N}(0, V) \)
    \end{enumerate}
    where \( V = (\nabla_\theta^2 \tilde{\ell}_\sigma(\hat{\theta}_\sigma;\mathcal{D}))^{-1} \mathbb{E}[\hat{S}(\hat{\theta}_\sigma; \mathcal{D}) \hat{S}(\hat{\theta}_\sigma; \mathcal{D})^\top] (\nabla_\theta^2 \tilde{\ell}_\sigma(\hat{\theta}_\sigma;\mathcal{D}))^{-1} \)
\end{theorem}

\subsubsection{Relationship between smoothed MLE and the true MLE}
Furthermore, we note here that we can establish an upper bound on the distance between the smoothed MLE \( \hat{\theta}_\sigma \) and the true MLE \( \hat{\theta} \). For simplicity, assume that we only have a single observation in our dataset, \( \mathbf{x} \). Then, using the strong concavity in Assumption~\ref{ass:smooth_con}, \(L\)-smoothness of the log-likelihood as in Theorem~\ref{thm:bias_variance-app}, and the result from \ref{appendix:FSM-BV} for a fixed point \( \mathbf{x} \), and denoting the gradient of the true log-likelihood as \( g(\theta; \mathbf{x}) = \nabla_\theta \ell(\theta; \mathbf{x}) \),

\begin{align*}
    \| \hat{\theta}_\sigma - \hat{\theta} \| &\leq \frac{1}{\mu} \| g(\hat{\theta}_\sigma; \mathbf{x}) - g(\hat{\theta} ; \mathbf{x}) \| = \frac{1}{\mu} \| g(\hat{\theta}_\sigma; \mathbf{x}) \| \\
    &= \frac{1}{\mu} \| g(\hat{\theta}_\sigma; \mathbf{x}) - S^*( \mathbf{x}; \hat{\theta}_\sigma) \| \\
    &= \frac{L \sigma \sqrt{d_\theta}}{\mu} \sup_\theta \frac{p(\theta \mid \mathbf{x}, \hat{\theta}_\sigma)}{q(\theta \mid \hat{\theta}_\sigma)}
\end{align*}

Thus, we have shown that \( \| \hat{\theta}_\sigma - \hat{\theta} \| \) is approximately of the order \( \mathcal{O}(\sigma) \).

\subsection{Uncertainty quantification of FSM}
\label{appendix:uq_proof}

Here, we show the uncertainty quantification by leveraging the result from Appendix~\ref{appendix:conv_guard} with classical MLE theory. As before, we denote \( \bar{\theta}_T = \frac{1}{T} \sum_{t=1}^T \theta_t\) as the averaged parameter iterate from the FSM-SGD procedure, and for clarity, we denote \( \hat{\theta}_{\mathrm{MLE},N} \) as the MLE of the true likelihood based on \( N \) i.i.d. observations. Our goal will be to characterize the distribution of \( \sqrt{N}( \bar{\theta}_T - \theta^*) \). 

First, we note that we can decompose \( \bar{\theta}_T - \theta^* \) into both algorithmic and statistical uncertainty,
\[
\bar{\theta}_T - \theta^* = \underbrace{(\bar{\theta}_T - \hat{\theta}_{\mathrm{MLE},N})}_{\text{algorithmic uncertainty}} + \underbrace{(\hat{\theta}_{\mathrm{MLE},N} - \theta^*)}_{\text{sampling uncertainty}}
\]

Multiplying by \( \sqrt{N} \), we obtain the following.
\[
\sqrt{N} (\bar{\theta}_T - \theta^*) = \sqrt{N}(\bar{\theta}_T - \hat{\theta}_{\mathrm{MLE},N}) + \sqrt{N}( \hat{\theta}_{\mathrm{MLE},N} - \theta^*)
\]

Focusing on the algorithmic error, observe that 
\[
\sqrt{N}(\bar{\theta}_T - \hat{\theta}_{\mathrm{MLE},N}) = \sqrt{\frac{N}{T}} \cdot\sqrt{T}( \bar{\theta}_T - \hat{\theta}_{\mathrm{MLE},N})  
\]

Since by assumption we know that \( \sqrt{\frac{N}{T}} \to 0 \) and \( X_{N,T} = \sqrt{T}( \bar{\theta}_T - \hat{\theta}_{\mathrm{MLE},N}) = \mathcal{O}_p(1) \) from Appendix~\ref{appendix:conv_guard}, this implies that their product is 
\[
\sqrt{\frac{N}{T}} \cdot X_{N,T} = \sqrt{N}(\bar{\theta}_T - \hat{\theta}_{\mathrm{MLE},N}) \to_p 0 
\]
as \( N,T \to \infty \).

From classical MLE theory, under standard regularity conditions, 
\[
\sqrt{N}( \hat{\theta}_{\mathrm{MLE},N} - \theta^*) \to_d \mathcal{N}(0, \mathcal{I}(\theta^*)^{-1}) 
\] 
where \( \mathcal{I}(\theta^*) \) is the Fisher information matrix at the true parameter.

Finally, to combine both results, using Slutsky's theorem, 
\[
\sqrt{N} (\bar{\theta}_T - \theta^*) \to_d  \mathcal{N}(0, \mathcal{I}(\theta^*)^{-1} ) 
\] 
as \( N,T \to \infty \)

\subsection{KDE-SP implementation}
\label{appendix:kde_sp}

We implement the KDE-SP gradient estimator as proposed in \citet{bertl2017approximate}, combining a kernel density estimate (KDE)-based likelihood approximation with a simultaneous perturbation stochastic approximation (SPSA). Specifically, at each iteration $t$, the approximate gradient of the log-likelihood at $\theta$ is given by:

\[
\hat{\nabla} \ell(\theta)=\delta_t \frac{\hat{\ell}\left(\theta^{+}\right)-\hat{\ell}\left(\theta^{-}\right)}{2 c_t},
\]

where $\theta^{+}=\theta+c_t \delta_t$ and $\theta^{-}=\theta-c_t \delta_t$ for a random perturbation $\delta_t$. This gradient estimate is then used in an SPSA update of the form

\[
\theta_t=\theta_{t-1}+\alpha_t \hat{\nabla} \ell\left(\theta_{t-1}\right)
\]

Following the specifications in \citet{bertl2017approximate}, we adopt the standard SPSA step size schedule:

\[
\alpha_t=\frac{a}{(t+A)^\alpha}, \quad c_t=\frac{c}{t^\gamma}
\]

with $\alpha=1, \gamma=1 / 6$, and $A=\lfloor 0.1 T\rfloor$, where $T$ is the total number of iterations. 

The constants $a$ and $c$ control the initial values of $\alpha_t$ and $c_t$. We tune both by performing a grid search over pairs ( $a, c$ ). For each candidate pair, we run a short trial of the SPSA optimization, simulate data from the resulting parameter estimates, and measure prediction error relative to the observed dataset. We then select the pair ( $a, c$ ) that yields the lowest validation error. We also incorporate the KDE modifications proposed in Section 3.2 of \citet{bertl2017approximate}, which refine the KDE-based likelihood approximation. These modifications help stabilize the KDE estimation for high-dimensional problems.

\subsection{Additional details and results on numerical studies experiment}
\label{appendix:exp_7.1}

\subsubsection{Additional results on hyperparameter sensitivity}

In Figures~\ref{fig:grad_comp_sim_full_2d} and ~\ref{fig:grad_comp_sim_full_20d}, we provide additional results on the ablation study showing the sensitivity of the gradient accuracy between the FSM method and the KDE-SP method for different choices of the proposal variance and perturbation constants, respectively. Figure~\ref{fig:grad_comp_sim_mini} is a subset of the results shown in Figure~\ref{fig:grad_comp_sim_full_2d}. Even in higher-dimensional settings, we find that the FSM method can match the gradient accuracy of the KDE-SP method across a wide range of hyperparameter choices.

\begin{figure}
    \centering
    \includegraphics[width=\linewidth]{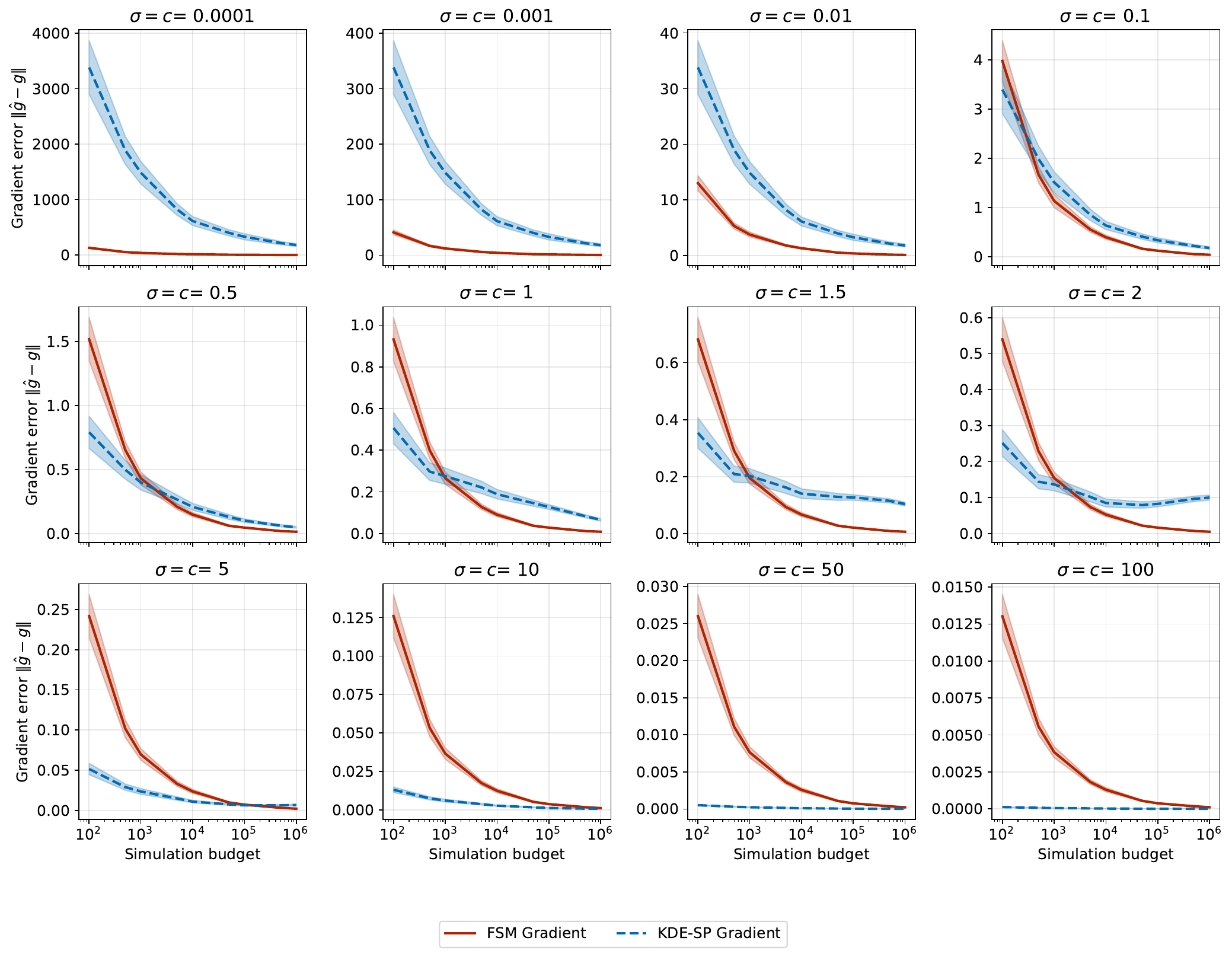}
    \caption{Gradient accuracy of both the Fisher score matching (FSM) technique and the KDE-SP method for a bivariate Gaussian likelihood for different choices of the proposal variance and perturbation constants. The error bars represent a 95\% confidence interval for 100 repeated gradient approximations.}
    \label{fig:grad_comp_sim_full_2d}
\end{figure}

\begin{figure}
    \centering
    \includegraphics[width=\linewidth]{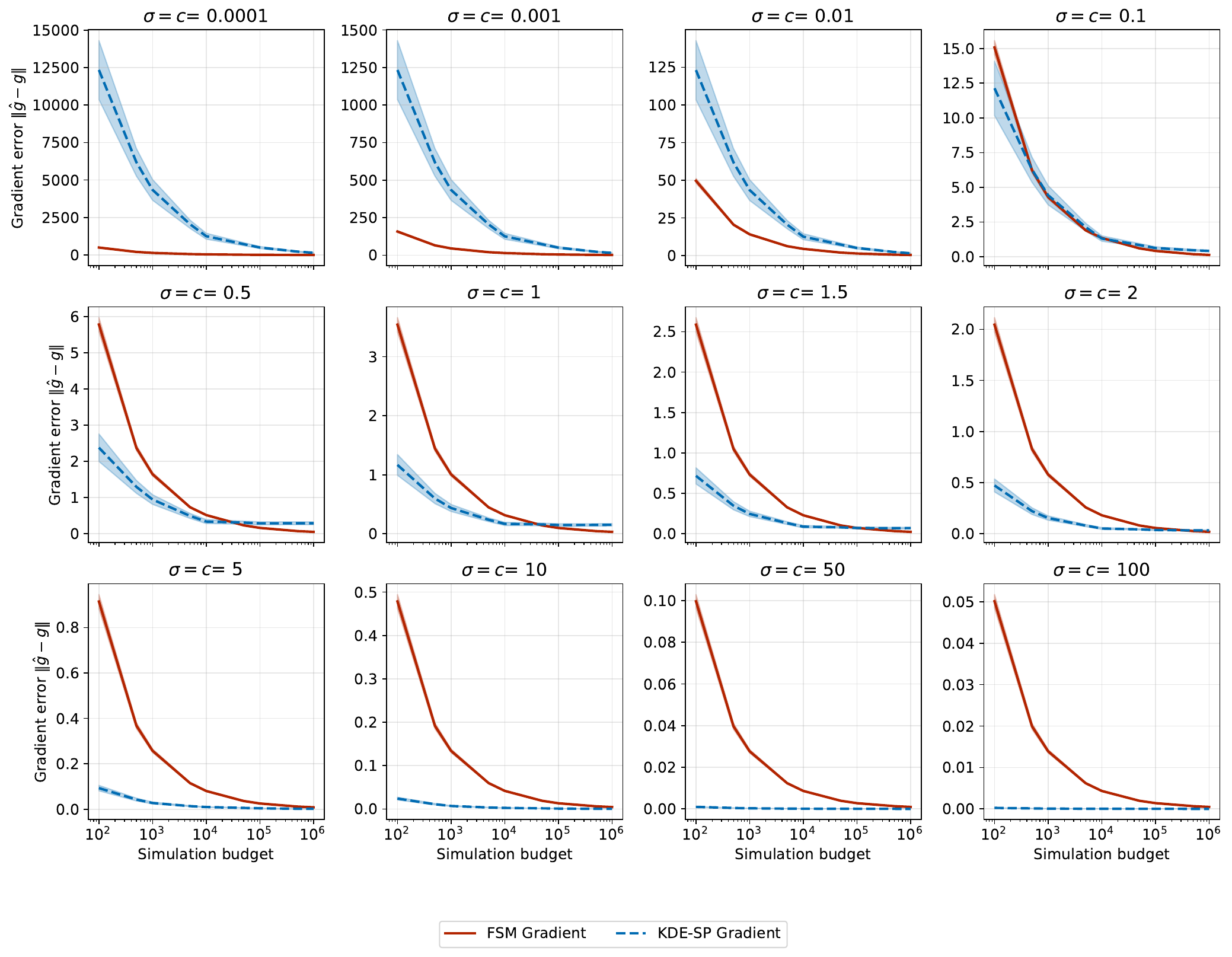}
    \caption{Gradient accuracy of both the Fisher score matching (FSM) technique and the KDE-SP method for a 20 dimensional Gaussian likelihood for different choices of the proposal variance and perturbation constants. The error bars represent a 95\% confidence interval for 100 repeated gradient approximations.}
    \label{fig:grad_comp_sim_full_20d}
\end{figure}

\subsubsection{Additional results on parameter dimension scaling}

Figure~\ref{fig:mvt_gaussian_param_dim_scaling_full} includes an additional result with the neural network FSM method for the parameter dimension scaling experiment seen in Figure~\ref{fig:mvt_gaussian_param_dim_scaling}. Furthermore, Figure~\ref{fig:mvt_gaussian_param_dim_all_scaling_full} shows the scaling with parameter dimension, with increasing simulation budgets, complementing the results seen in Figure~\ref{fig:mvt_gaussian_param_dim_scaling_full}. Generally, we find that the linear FSM method performs the best across different parameter dimensions and simulation budgets. The KDE-SP method performs worse in higher dimensions, likely due to the curse of dimensionality affecting the kernel density estimate. The neural network FSM method shows competitive performance in lower dimensions, but its performance quickly degrades in higher dimensions, possibly due to optimization and/or overfitting issues.

\begin{figure}
    \centering
    \includegraphics[width=\linewidth]{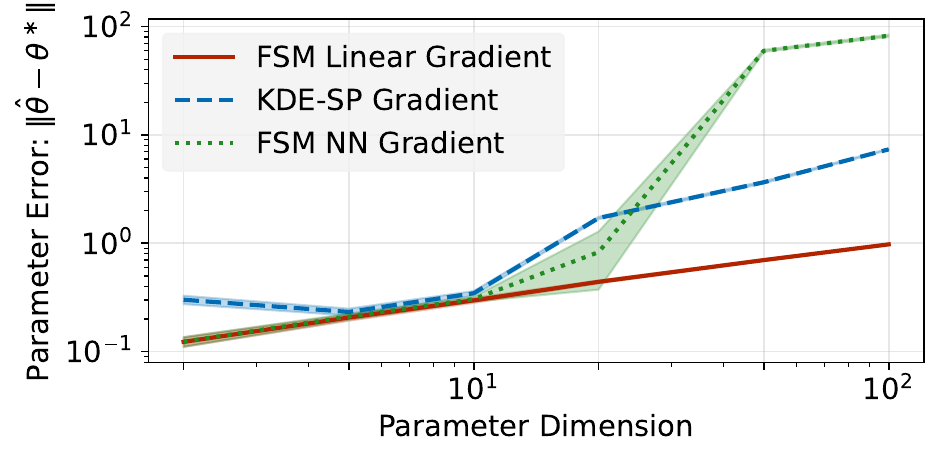}
    \caption{Parameter estimation accuracy of the Linear FSM, Neural Network FSM, and KDE-SP methods under increasing parameter dimensions, over 100 repeated optimization runs.}
    \label{fig:mvt_gaussian_param_dim_scaling_full}
\end{figure}

\begin{figure}
    \centering
    \includegraphics[width=\linewidth]{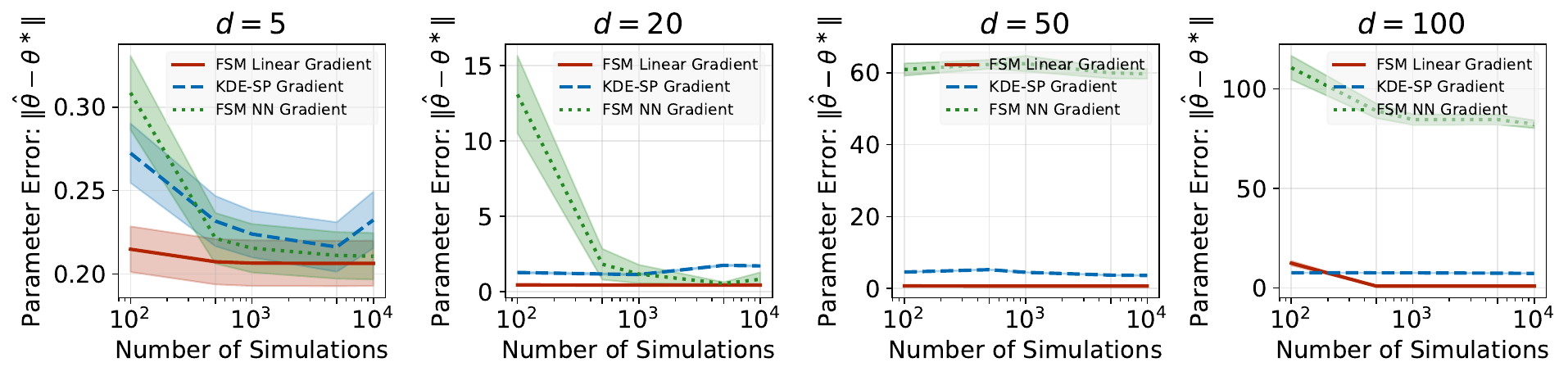}
    \caption{Parameter estimation accuracy of the Linear FSM, Neural Network FSM, and KDE-SP methods under increasing parameter dimensions and increasing simulation budgets, over 100 repeated optimization runs.}
    \label{fig:mvt_gaussian_param_dim_all_scaling_full}
\end{figure}

\subsubsection{Additional results on wall-clock time}
We provide a comparison of the wall-clock time in Figures~\ref{fig:fig-8} and ~\ref{fig:fig-9} for repeated gradient estimation procedures for both the KDE-SP and FSM methods. As we can see in Figure~\ref{fig:fig-8}, the FSM scales favorably with respect to the increase in the number of simulation budgets. However, in Figure~\ref{fig:fig-9}, the matrix inversion step of the linear FSM method grows cubically with the parameter dimension, and hence causes an increase in the wall-clock time for the FSM method. We note that in practice one can reduce this cost considerably by employing faster linear solvers (e.g., conjugate gradient methods), which can greatly improve scalability in higher dimensions.

\begin{figure}
    \centering
    \includegraphics[width=0.5\linewidth]{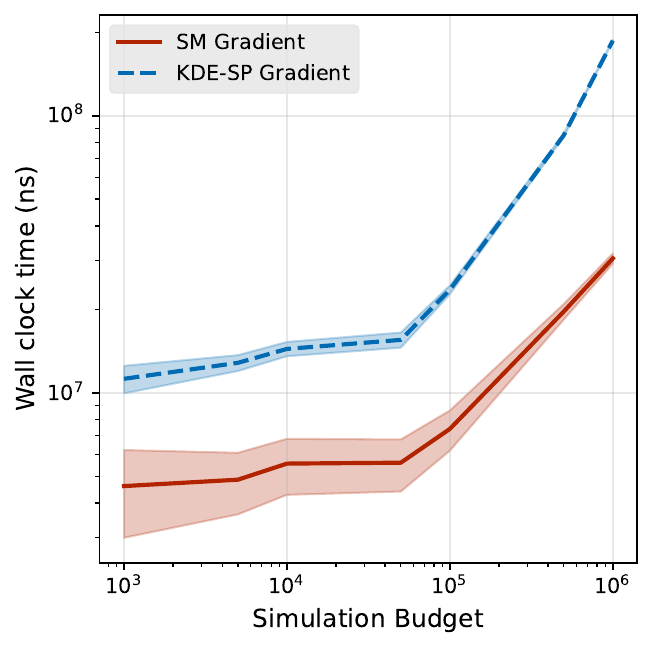}
    \caption{Wall clock time comparison between FSM and KDE-SP estimation, over 1000 runs for increasing simulation budgets}
    \label{fig:fig-8}
\end{figure}

\begin{figure}
    \centering
    \includegraphics[width=0.5\linewidth]{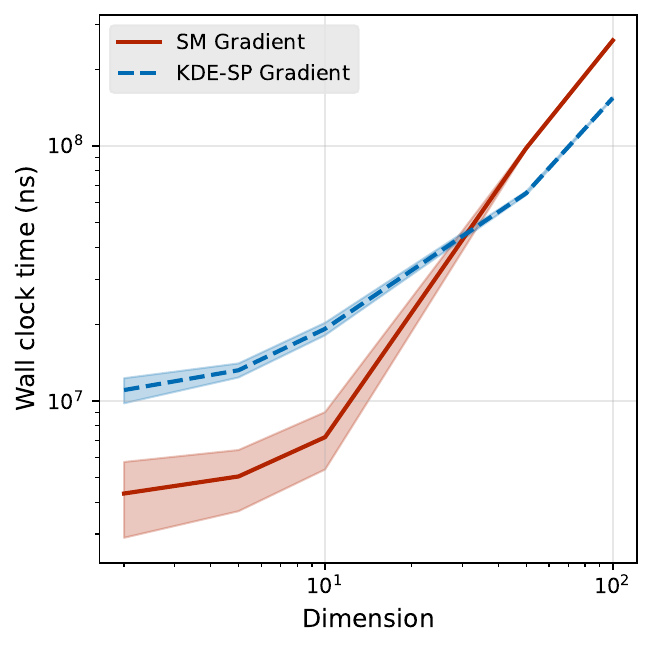}
    \caption{Wall clock time comparison between FSM and KDE-SP estimation, over 1000 runs for increasing parameter dimension}
    \label{fig:fig-9}
\end{figure}

\subsubsection{Additional results on confidence interval construction}
We also note that in Figure~\ref{fig:fig-10}, we provide a simple validation test for the use of the FSM estimate for the Fisher information matrix estimation. This shows that we can recover a well-calibrated confidence interval even with the use of a stochastic Fisher score estimate.

Further details of these experiments are provided in the Appendix~\ref{appendix:A.9.3}.

\begin{figure}[H]
    \centering
    \includegraphics[width=0.5\linewidth]{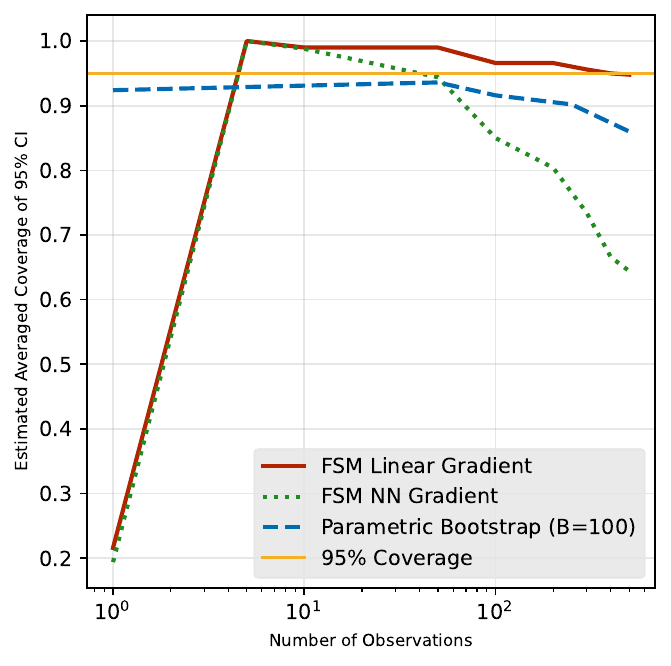}
    \caption{Estimated coverage of constructed confidence interval (averaged across all parameter dimensions) from the approximated Fisher information matrix estimation with FSM estimates, comparing repeating the FSM estimation procedure with nonparametric bootstrap}
    \label{fig:fig-10}
\end{figure}

\subsubsection{Experimental details}
\label{appendix:A.9.3}
The gradient comparison experiment corresponding to the plots of Figures~\ref{fig:grad_comp_sim_mini} and \ref{fig:grad_comp_sim_full_2d} was carried out for a bivariate Gaussian mean model with \(10\) observations. Observations were generated from a true mean of \( (1.0, 1.0) \) (Figure~\ref{fig:grad_comp_sim_full_20d} used a 20-dimensional Gaussian with the same setting), and Fisher score estimates were taken at the observation means, which is also the maximum likelihood estimator. The uncertainty was obtained by repeating \( 100 \) runs of the score estimation for both methods.

The multivariate Gaussian parameter estimation accuracy in Figure~\ref{fig:mvt_gaussian_param_dim_all_scaling_full} was performed with \( 100 \) observations, and with parameter dimensions of \( d=5, 20, 50, 100\) for \( 100 \) optimization steps, using \( 100 \) repeated runs as with the previous experiment. Figures~\ref{fig:mvt_gaussian_param_dim_scaling} and \ref{fig:mvt_gaussian_param_dim_scaling_full} were performed in a similar way, by fixing a total simulation budget of \( 1000 \) and increasing parameter dimensions of \( d=2, 5, 10, 20, 50, 100 \). The true parameters used to generate the observations were similarly taken to be a vectors of ones as with the previous experiment. The \( (a,c) \) hyperparameters for the KDE-SP gradient method was selected from a grid of \( [10^{-2}, 10^{-1}, 10^{0}, 10^{1}, 10^{2}, 10^{3}] \times [10^{-2}, 10^{-1}, 10^{0}, 10^{1}, 10^{2}, 10^{3}] \). For the FSM-based estimation, the \( (\sigma, \eta) \) hyperparameters, corresponding to the proposal variance and step size, were tuned in the exact same way as the KDE-SP gradient hyperparameters (using the prediction error), but over a grid of \( [10^{-3},10^{-2}, 10^{-1}] \times [10^{-2}, 10^{-1}, 10^{0}] \) instead. The Adam \citep{kingma2017adam} optimizer was used for the FSM-based estimation, with averaging over the last 50 iterations of the parameter iterates.

For the wall-clock time comparisons in Figures~\ref{fig:fig-8} and \ref{fig:fig-9}, each gradient estimation procedure was timed for \( 1000 \) runs on a bivariate Gaussian mean model with \( 10 \) observations. As both the FSM and KDE-SP gradient estimation was implemented in Python and the \textsc{jax} package, best attempts were made to equalize the comparison between the two methods. All just-in-time (JIT) compilations for both methods were disabled for the wall-clock tests to remove compilation overhead.

For the confidence interval experiment of Figure~\ref{fig:fig-10}, a \( 5\) dimensional multivariate Gaussian mean model was used. A step size of \(10^{-3}\) with \( \sigma = 0.05 \) was used with the RMSProp \citep{tieleman2012lecture} optimizer. The final Fisher information matrix was estimated by simulating \( 100 000\) simulations from the resulting MLE estimate of the optimization run, which was used to construct the confidence interval. This was repeated for \( 100 \) runs to obtain an estimated coverage probability.

For the neural network-based FSM method, a standard feedforward neural network with two hidden layers of size \( 16 \) with ReLU activations was used.  Adam optimizer with a step size of \( 10^{-2} \) was used to train the neural network for \( 10 \) iterations, for each parameter iteration of the MLE optimization procedure.

All experiments in this section were performed on a standard consumer laptop, an Intel i7-11370H CPU with \( 64 \)GB of RAM.

\subsection{Additional details on LSST weak lensing experiment}
\label{appendix:exp_7.2}

For the weak lensing experiment in Figure~\ref{fig:fig-11}, \(100 \) iterations of the gradient optimization method were used with both the KDE-SP and FSM estimators, with \( 100 \) simulations per iteration, giving a total simulation budget of \( 10000\) simulations for the entire optimization process. The dimension of the parameter space is \( 6 \), and the dimension of the summary statistics used is \(6 \) as well.

The same amount of simulations was provided to a neural likelihood estimater, which is a standard masked autoregressive flow model in the package \textsc{sbi} in Python \citep{BoeltsDeistler_sbi_2025}. To mimic a general, uninformative prior, we used the priors for the parameters provided in Table 1 of \citet{zeghal2024simulationbasedinferencebenchmarklsst}, which are all Gaussian priors, and converted them to a uniform prior by taking three standard deviations from the mean, \( \mathcal{U}[\mu - 3*\sigma, \mu + 3*\sigma] \), where the original Gaussian priors are represented as \( \mathcal{N}(\mu, \sigma^2) \). The NLE was trained with \( 10000 \) (parameters, data) pairs drawn from this prior, and \( 5000 \) iterations with a standard Adam optimizer were used to train the NLE. To optimize the NLE for a specific observational dataset, we evaluated the trained NLE at the specific dataset, and directly differentiated through the NLE model, giving us a deterministic gradient, which is used in a standard gradient-based optimization procedure. The likelihood is optimized until convergence, where there is no longer any change in the estimated likelihood with the NLE. 

The hyperparameters \( (a,c) \) for the KDE-SP gradient method were selected from a grid of \( [10^{-5}, 10^{-4}, 10^{-3}, 10^{-2}] \times [10^{-3}, 10^{-2}, 10^{-1}, 10^{0}]\). For the FSM method, we set \( \sigma =10^{-3} \) and a step size of \( 10^{-2} \), with parameter averaging over the final \( 50 \) iterations.

For the neural network-based FSM method, a standard feedforward neural network with two hidden layers of size \( 16 \) with ReLU activations was used.  Adam optimizer with a step size of \( 10^{-2} \) was used to train the neural network for \( 10 \) iterations, for each parameter iteration of the MLE optimization procedure. We find that the neural network-based FSM method did not perform well in this experiment, often suffering from high variance and instability during training.

An RTX 4090 GPU with 24GB of VRAM, 41GB of RAM was used in this experiment.

\begin{figure}
    \centering
    \includegraphics[width=\linewidth]{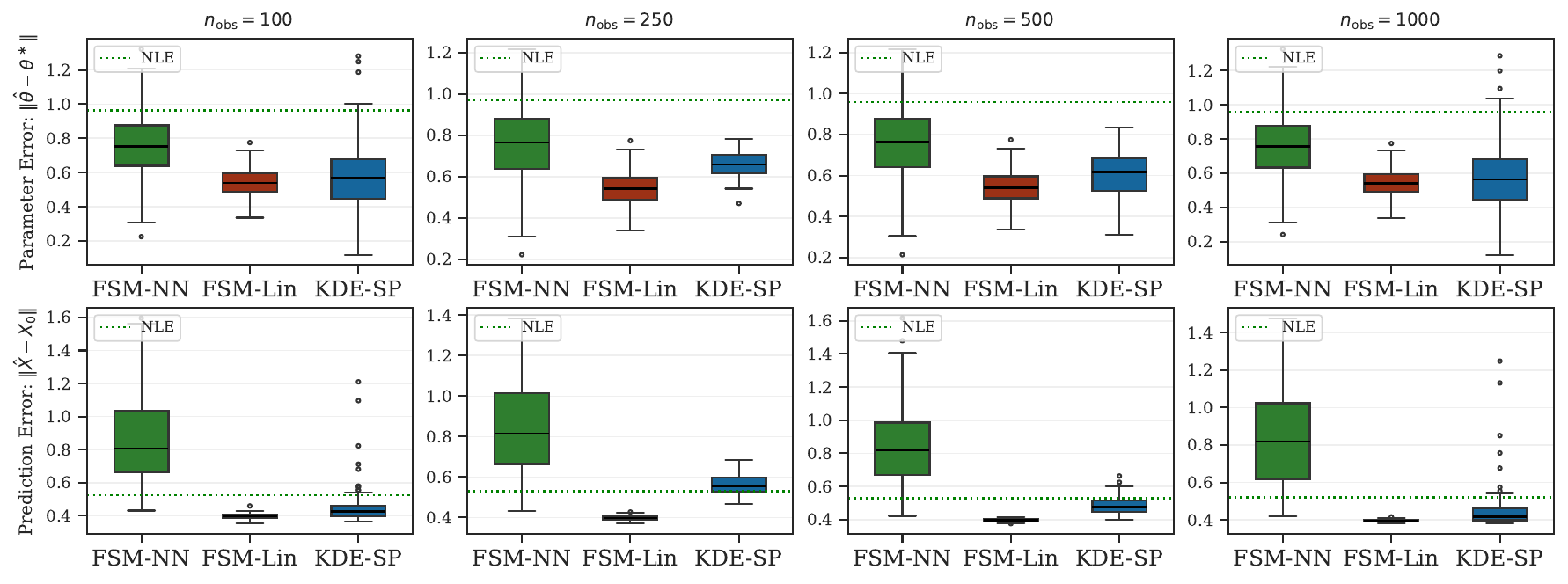}
    \caption{Parameter estimation and prediction accuracy of the NLE, FSM (linear and neural-network based) and KDE-SP methods for the LSST-Y10 weak lensing model, for increasing number of observations}
    \label{fig:fig-11}
\end{figure}

\subsection{Additional details on generator inversion task}
\label{appendix:exp_7.3}

For the generator inversion task in Figure~\ref{fig:fig-12}, we trained a standard GAN on a down-scaled \(16 \times 16\) MNIST dataset, giving a data dimension of \( 256 \) as no summary statistics were used. We used \( 500 \) iterations of the gradient optimization method with both the KDE-SP and the FSM gradient estimation procedure, with \( 22500 \) simulations per parameter iteration used in the gradient estimation. The dimension of the parameter space is \( 50 \).

The \( (a,c) \) hyperparameters for the KDE-SP gradient method was selected from a grid of \( [10^{-4}, 10^{-3}, 5\times10^{-3}, 10^{-2}, 5 \times 10^{-2}] \times [10^{-4}, 10^{-3}, 5\times10^{-3}, 10^{-2}, 5 \times 10^{-2}]\). For the FSM method, we set \( \sigma =0.2 \) and a step size of \( 5 \times10^{-2} \), with parameter averaging over the last \( 300 \) iterations. The latent mean prior, \( \sigma_z\) was set at \( 0.1\).

The direct optimization approach was performed by directly minimizing a reconstruction loss (mean squared error in pixel space) between the generated images and the observations, and directly differentiating through the generator network \( G_w \). Specifically, we minimize the following loss function.
\[
\min_\theta \mathcal{L}(G_w(\theta), \mathbf{x}_0) = \frac{1}{n}\sum_{i=1}^n \|G_w(\mathbf{z}_i) - \mathbf{x}_0 \|^2
\]
where \( \mathbf{z}_i \sim \mathcal{N}(\theta, \sigma_{\mathbf{z}}^2 I) \). This is done with the Adam optimizer with a step size of \( 5 \cdot 10^{-1}\), and for \( 1000 \) iterations, with \( n=100\) simulations per iteration.

For the neural network-based FSM method, a standard feedforward neural network with two hidden layers of size \( 16 \) with ReLU activations was used.  Adam optimizer with a step size of \( 10^{-3} \) was used to train the neural network for \( 10 \) iterations, for each parameter iteration of the MLE optimization procedure. Compared to Appendix~\ref{appendix:exp_7.2}, we found that the neural network-based FSM method performed better in this experiment, but still generally had subpar performance compared to the linear FSM method and with increased variance.

An RTX 4090 GPU with 24GB of VRAM, 41GB of RAM was used in this experiment.

\begin{figure}
    \centering
    \includegraphics[width=\linewidth]{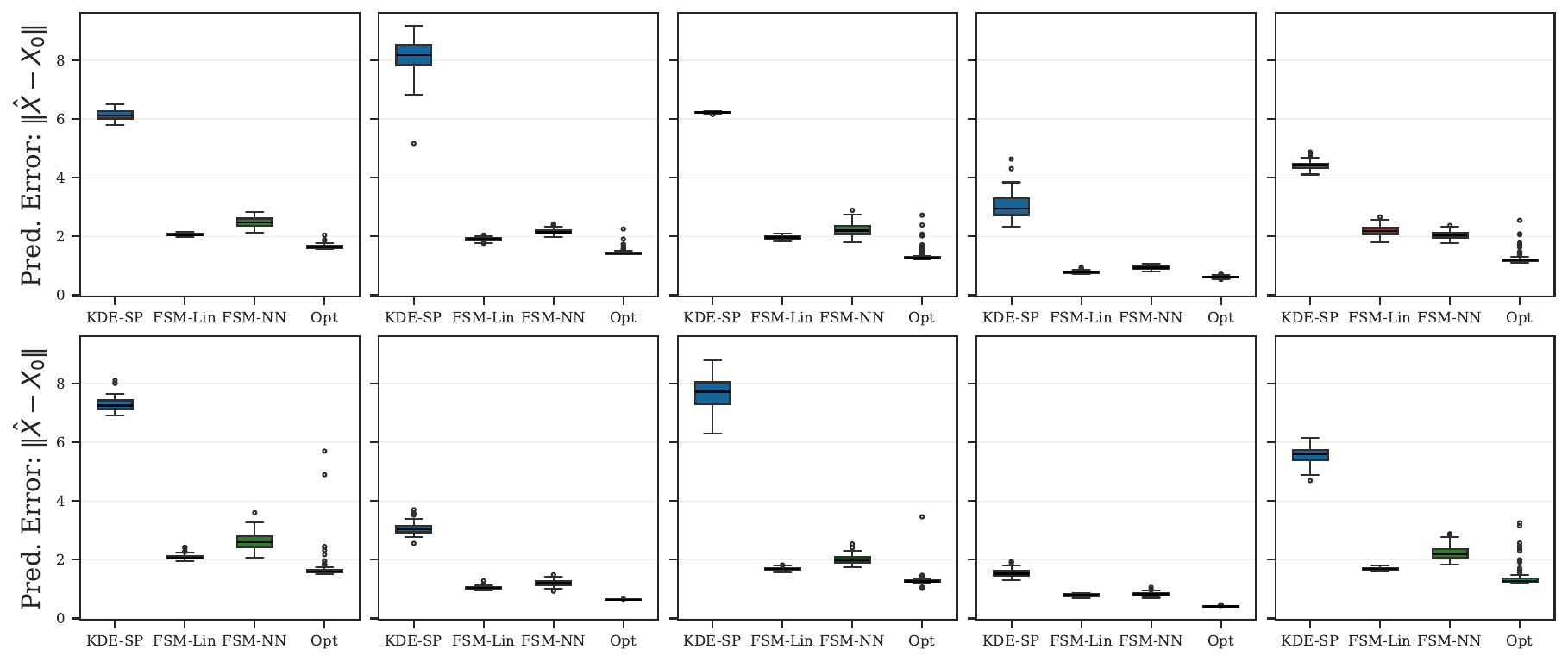}
    \caption{Prediction error for the FSM (linear and neural-network based), KDE-SP and direct optimization method for the latent GAN inversion, each boxplot corresponds to a different observation}
    \label{fig:fig-12}
\end{figure}

\newpage

\newpage

\end{document}